\documentclass{article}

\usepackage{enumitem,wrapfig,caption}
\usepackage{fullpage}

\usepackage[utf8]{inputenc} 
\usepackage[T1]{fontenc}    
\usepackage{hyperref}       
\usepackage{url}            
\usepackage{booktabs}       
\usepackage{amsfonts}       
\usepackage{nicefrac}       
\usepackage{microtype,xfrac}      

\usepackage{hyperref,color}
\usepackage{url}

\usepackage{times}
\usepackage{graphicx} 
\usepackage{subfigure}

\usepackage{algorithmic}

\usepackage{hyperref}

\usepackage{amsthm}
\newtheorem{thm}{Theorem}[section]
\newtheorem{lem}[thm]{Lemma}
\newtheorem{prop}[thm]{Proposition}
\newtheorem{cor}[thm]{Corollary}
\newtheorem{rem}[thm]{Remark}
\newtheorem{defn}{Definition}[section]

\newcommand{\eps}{\varepsilon}

\renewcommand{\hat}{\widehat}
\renewcommand{\tilde}{\widetilde}
\newcommand{\defeq}{\stackrel{{\rm def}}{=}}

\newcommand{\F}{\mathcal F}
\newcommand{\A}{\mathcal A}

\newcommand{\E}{{\mathbb E}}

\newcommand{\EXPIX}{\mathrm{\sc EXP.IX }}
\newcommand{\BEXP}{\mathrm{\sc BEXP }}

\newcommand{\MUA}{\mathrm{\sc FULL }}
\newcommand{\MUM}{\mathrm{\sc MUA }}
\newcommand{\EXPThree}{\mathrm{\sc EXP3 }}
\newcommand{\EXPG}{\mathrm{\sc EXP3G }}
\newcommand{\UCB}{\mathrm{\sc UCB }}
\newcommand{\UCBN}{\mathrm{\sc UCBN }}
\newcommand{\GOB}{\mathrm{\sc GOB.LIN }}

\newcommand{\MABN}{\mathrm{\sc EXPN }}
\newcommand{\EXPN}{\mathrm{\sc EXPN }}
\newcommand{\GOBLIN}{\mathrm{\sc GOB.LIN }}

\usepackage{amsmath}
\usepackage{amsfonts}
\usepackage{amsopn}
\usepackage{ifthen}
\usepackage{xfrac}

\begin{document}

\title{Lean From Thy Neighbor: \\ Stochastic \& Adversarial Bandits in a Network}

\author{ L. Elisa Celis \\ elisa.celis@epfl.ch \\
       \'{E}cole Polytechnique F\'{e}d\'{e}rale de Lausanne (EPFL)\\$\;$\\
 Farnood Salehi \\   farnood.salehi@epfl.ch \\
       \'{E}cole Polytechnique F\'{e}d\'{e}rale de Lausanne (EPFL)
       }
       

\maketitle


\begin{abstract}
An individual's decisions are often guided by \emph{those of his or her peers}, i.e., neighbors in a social network. Presumably, being privy to the experiences of others aids in learning and decision making, but how much advantage does an individual gain by observing her neighbors? Such problems make appearances in sociology and economics and, in this paper, we  present a novel model to capture such decision-making processes and appeal to the classical multi-armed bandit framework to analyze it. Each individual, in addition to her own actions, can observe the actions and rewards obtained by her neighbors, and can use all of this information in order to minimize her own regret. We provide algorithms for this setting, both for stochastic and adversarial bandits, and show that their regret smoothly interpolates between the regret in the classical bandit setting and that of the full-information setting as a function of the neighbors' exploration. In the stochastic setting the additional information must simply be incorporated into the usual estimation of the rewards, while in the adversarial setting this is attained by constructing a new unbiased estimator for the rewards and appropriately bounding the amount of additional information provided by the neighbors. These algorithms are optimal up to log factors; despite the fact that the agents act independently and selfishly, this implies that it is an approximate Nash equilibria for all agents to use our algorithms. Further, we show via empirical simulations that our algorithms, often significantly, outperform existing algorithms that one could apply to this setting. 
\end{abstract}


\section{Introduction}

Individuals often have access to {information}, via their social or economic network, that they can use to make improved decisions. 
This phenomenon has been observed widely in the social and natural sciences. 
For instance, a recent work (\cite{Y2012}) studies farmers who, every year, have to decide which kind of seed to plant (not just what kind of crop, but which variety of seed) in order to attain the most profit (i.e., revenue - cost). 
In their study, \cite{Y2012}  finds that farmers' decisions are based on a) their own experience in previous years of how different varieties performed, and b) the experiences of peers attained either directly (explicitly via conversations with social contacts) or indirectly (implicitly by observing the farming practices of peers).
Moreover, the information farmers used is primarily from peers in their {\em physical neighborhood} -- not only because these are where their contacts are most likely to be, but also because the profit is correlated due to similar soil and weather conditions. 
These connections between peers then form a network of farmers across the country, where locally, each farmer is trying to learn the best seed for their farm using her own information and that of her neighbors.
As another example, consider WI-FI networks in which nodes want to send their data across the best frequency band. 
Nodes could obtain the current quality of the band their peers are using indirectly through capacity estimation or directly by message passing, and use this information to determine which band to use.
Similar social learning phenomena appear in many other areas in various disguises -- e.g., in the acquisition of consumer products by individuals, the adoption of new technologies, the prevalence and spread of corruption, and in the behavior of animals such as squirrels; see \cite{LBG1948,KL1955,ZAA2007,S2006,AS2012}. 
A natural question then is: \emph{How should an individual incorporate the  information from their neighbors in order to make the best decisions, and how much improvement can such information bring?} 

\begin{wrapfigure}{l}{.5\textwidth} 
\vspace{-.1in}
\begin{center}
\includegraphics[width=2.75in]{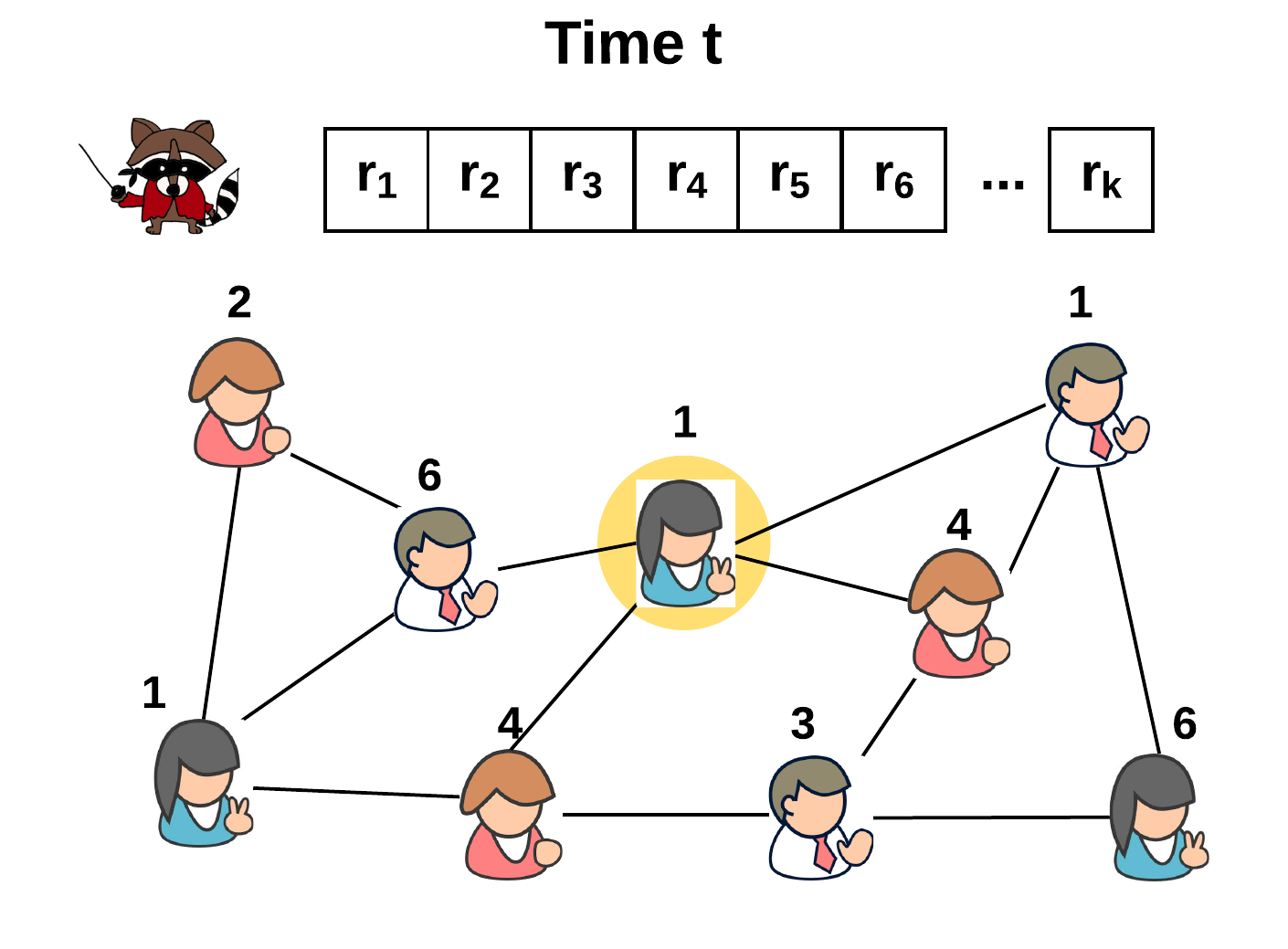}
\vspace{-.1in}
\caption{\small{A social network in which all individuals play against the same bandit, i.e., if two individuals select the same arm at the same time step, they observe the same reward (up to noise). 
 At each time $t$, each individual selects an arm (shown), and then observes its reward along with the actions/rewards of her neighbors. E.g., the yellow circled individual would observe the reward of actions 1, 4 and 6 in this time step.}}
\label{fig:model}
\end{center}
\vspace{-.1in}
\end{wrapfigure}

Consider the following  formulation geared towards capturing the type of settings mentioned above 
(see also Figure~\ref{fig:model}): at each time step each individual selects one of $K$ possible actions,  
observes the value or reward of selecting that action, and observes the actions, values and/or decision process of their neighbors in the network. 
This selection and observation is repeated again and again, and each individual has the end goal of identifying the action $a^\star \in K$ that brings them the best value over all time steps; i.e., minimizing the \emph{regret}.  
This formulation seems to suggest that the problem is suited for study using the bandit optimization framework, except that now there is \emph{additional information available to an individual via her neighbors}.

Towards this, one approach could be to consider the framework of bandits with \emph{side observations} for which, in the adversarial setting, variants of the multiplicative-weight update algorithm have been developed with success. 
Informally, side observations just mean that at each time step, in addition to observing the reward of a selected action $a(t)$, one may {observe} (but not receive) rewards from a set of other actions $S(t)$. 
A recent body of work has explored how to minimize regret for various different models of $S(t)$. 
In the \emph{free observation model}, the individual is allowed to select $S(t)$ up to some cardinality (e.g., \cite{AKTT2015}).  
However, if one tried to apply such algorithms to the social settings  considered above, it would require an individual to decide which actions her neighbors should take, and hence is not feasible as a solution in this setting. 
In another line of  work (e.g., \cite{ACDK2015}) an \emph{action-network model} has been studied: 
Here, the {actions} form a network and the individual observes the rewards of the neighbors of the \emph{action} she selects (as opposed to the rewards of the actions that \emph{her} neighbors select). 
The action-network is often taken to be exogenous and can be changing over time. Thus, one can apply the algorithms developed in the action-network setting to the social setting above by defining $S(t)$ to be the set of actions selected by the individual's neighbors; however, this may not always be optimal for the social setting as neighbors can provide even more information (see {Section~\ref{sec:better}}). 
The above approaches have been developed in independent contexts and hence geared towards different settings.
Towards obtaining optimal results in the social setting described above, 
the challenge is to adequately model the information from neighbors that can aid in learning, leverage it appropriately, and quantify the advantage it provides. 
In this paper we present such a model 
and show how one can incorporate the additional social information in order to obtain optimal results.
More specifically, in the stochastic setting we show that incorporating side information in a simple way gives rise to a near-optimal algorithm.
Furthermore, in the adversarial setting, we present a modified multiplicative-weight update algorithm that uses an new unbiased estimator to incorporate this information appropriately into the estimation of the reward of each action. 
The proofs requires us to overcome some additional hurdles in order to bound the amount of information gleaned from the neighbors and attain optimal bounds on the regret. 
Our algorithms outperform other state-of-the-art bandit algorithms both theoretically (Section~\ref{sec:adv}) and empirically (Section~\ref{sec:empirical}). 
%


\subsection{Summary of Our Results}
\label{sec:ourmodel}

Since our model appeals to the bandit  framework, we start by recalling its salient features: 
In the bandit optimization setting, there are $K$ potential actions (\emph{arms}), and the individual selects one arm at each time step. 
Each arm $j$ has an (unknown) \emph{reward} $g_j(t)$ at \emph{time} $t = 1, \ldots, T$. 
The individual receives the reward $g_{a(t)}(t)$ for the selected arm $a(t)$, while the rewards $g_j(t)$ for all other actions $j \neq a(t)$ remain unknown. 
Ideally, the individual would like to select the arm with the best overall reward, i.e., 
$ j^\star = \mbox{arg}\max_j \E\left[ \sum_{t=1}^{T}g_j(t)\right]. $
However, in the absence of knowledge about the reward structure, this is not feasible. 
Instead, as is prevalent in the online learning literature, $j^\star$ is used as a \emph{benchmark} and the \emph{regret} $R$ (i.e., the difference between the individual's rewards and those of the best arm) is minimized. 
Formally,  
\[R \defeq \E\left[\sum_{t=1}^{T} g_{j^\star}(t) - \sum_{t=1}^{T} g_{a(t)}(t)\right],\]
where the expectation is taken over the randomness in the rewards (if stochastic) and the randomness in the algorithm.\footnote{{This is often referred to as \emph{pseudo-regret}; in this paper we simply refer to it as \emph{regret}.}} 
An important divide arises from how the rewards are decided: 
for \emph{stochastic bandits}, the rewards $g_j(t)$ are  i.i.d. from an (unknown) distribution $\F_j$, and for 
\emph{adversarial bandits} (also known as non-stochastic bandits), the rewards $g_j(t)$ are set in an arbitrary manner by an \emph{adversary} that knows the individual's algorithm and past coin flips. 
In realistic settings, such as the examples mentioned in the introduction, the situation is often somewhere in-between. 
Thus, in this work we consider both extremes by analyzing both the stochastic and adversarial settings. 
In either case, algorithms must carefully tradeoff between \emph{exploration} (gaining new knowledge about the rewards) and \emph{exploitation} (using current knowledge to maximize rewards) in order to minimize regret. 
In either case, algorithms must carefully tradeoff between \emph{exploration} (gaining new knowledge about the rewards) and \emph{exploitation} (using current knowledge to maximize rewards) in order to minimize regret.

\paragraph{\textnormal{\textbf{The Model.}}}

Firstly, we assume all individuals play against the same multi-armed bandit: 
in the stochastic setting, the reward distributions $\F_i$ is the same for arm $i$ for all individuals (although the realizations at any given time may differ), and in the adversarial setting the reward vector $\mathbf{g}(t)$ selected by the adversary at time $t$ is the same for all individuals. 
Clearly, there must be some similarity in the rewards between neighbors for social learning to occur.  
Our results also extend to the setting in which the distributions or rewards are correlated, e.g., via 0-mean noise (i.e., each individual $i$ receives the reward as above + individual noise); for ease of presentation we omit this extension, the proofs follow analogously.  %
Secondly, we assume that each individual can observe the following for all \emph{neighbors} $i$: 
\begin{enumerate}[label={(\arabic*)},topsep=1pt,itemsep=1pt,partopsep=0ex,parsep=0ex,,leftmargin=.75in]
\item the actions $a_i(t)$,
\item the rewards $g_{a_i(t)}(t)$, and
\item (for the adversarial setting only) the probability distribution each neighbor used to select an arm at \emph{the previous} time step. %
\end{enumerate}
Assumptions (1) and (2) are natural and directly inspired by applications such as those mentioned in the introduction; individuals either directly or indirectly observe their neighbors' actions and rewards. 
(3) additionally assumes a limited knowledge of how neighbors made their decisions on a step-by-step basis, without having to assume we know their overall algorithm or restricting their behavior in any way. 
All individuals are free to select their probability distributions arbitrarily (and depend arbitrarily on each other), and each one can draw her decision independently of the rest. 
While it would be nice to drop assumption (3) entirely, this would prevent us from attaining optimal regret bounds (see Proposition~\ref{thm:arm_lb} and the discussion in Appendix~\ref{sec:comparison}).\footnote{Alternatively, under different assumptions (e.g., if we assume the neighbors are using an algorithm such as $\EXPThree$) we can estimate these distributions which would suffice.}

Importantly, an individual 
\begin{enumerate}[label={(\arabic*)},topsep=1pt,itemsep=1pt,partopsep=0ex,parsep=0ex,leftmargin=.75in] 
\setcounter{enumi}{3}
\item  \emph{does not} know about the actions and rewards of individuals beyond her neighbors, 
\item \emph{does not} know any global properties of the network, 
\item \emph{does not} know which algorithm other individuals (including neighbors) are using, and 
\item \emph{cannot} dictate or coerce other individuals to act a certain way. 
\end{enumerate}
Removing any of assumptions (4)-(7) would be unnatural for the social learning setting described above: 
If (4) does not hold, we would simply consider such an individual a neighbor, removing (5) is unnatural as the network can be very large and we cannot expect to have knowledge of distant individuals, removing (6) seems impractical as it would mean that individuals  have a detailed knowledge of how neighbors select actions, and allowing individuals to be coerced as in (7) would be in conflict with the idea that every individual seeks to improve her own performance.
Hence, given (1)-(7), any improvement in the individual's regret arises solely from \emph{passive} observation of \emph{local} information.

\begin{rem}
We emphasize that, by setting up the model as above, agents can act independently and work to selfishly minimize their own regret. Hence, (as our algorithms are near-optimal) it is always an individual's (approximate) best-response to use our algorithms, \emph{regardless of what her neighbors do}. Putting it another way, when the action space is the set of algorithms, it is an (approximate) Nash equilibrium for all nodes to use ours -- no individual can significantly improve their regret by deviating to use an alternate algorithm. This gives rise to an interesting set of questions regarding the average regret of a network (depending on its structure) in equilibrium, or the expected regret in equilibrium of a node as a function of its position in the network. Some immediate implications about the properties of such equilibria follow directly from our results.
\end{rem}

\paragraph{The Algorithm in the Stochastic Setting.} 

Algorithms for the classical stochastic bandit setting \cite{agrawal2012analysis, ACF2002, GC2011} maintain metrics about the observed samples, and determine which arm to select based on these metrics. 
Hence, a natural strategy would be for an individual in the networked setting to take one of these algorithms and incorporate samples obtained from neighbors along with her own without differentiating between the two. 
Indeed, we this simple insight suffices to get near-optimal results. 
Our $\UCBN$ algorithm extends the classic $\UCB$ algorithm by incorporating all observed samples indiscriminately. We show that this suffices to improve performance, often dramatically, both asymptotically and in silico. These results are presented in Section~\ref{sec:tech_stoch} and Appendix~\ref{sec:stochastic}. 
We show that the regret of our algorithm interpolates between $O(1)$ and the $O(K\ln T)$ regret for the classic bandit setting  
depending on the amount of exploration conducted by the neighbors (see {Theorem~\ref{thm:UCBN_ub}}), and  
these bounds are asymptotically optimal (see {Theorem~\ref{thm:UCBN_LB}}). 
As a corollary, in the complete network with $b$ vertices, if all neighbors use our algorithm (i.e., in equilibrium) the regret reduces to $O(\frac{K \ln T}{b})$, which is optimal as even a fully centralized approach can speed up learning by at most a factor of the increase in the number of available samples, i.e., $b$ (see {Corollary~\ref{allUCBN}}).
The theoretical results are presented in {Section~\ref{sec:tech_stoch}} and Appendix~\ref{sec:stochastic}), and the empirical results in {Section~\ref{sec:empirical_stoch}}).

\paragraph{\textnormal{\textbf{The Algorithm in the Adversarial Setting.}}}

Algorithms for classic adversarial bandits are, typically, variations of the multiplicative weights update method. 
Such algorithms maintain a weight for each arm, and (multiplicatively) update that weight based on the observed reward(s). 
An individual then selects an action proportionally to the weight vector. 
It is easy to verify that if we na\"ively incorporate samples from neighbors as if they were our own into the weight vector there is no improvement in the regret. 
Our variation of the multiplicative-weights update algorithm a) incorporates information in a clever manner (that allows for new improved bounds) by using a different unbiased estimator for the rewards and then b) adapts according to the behavior of its neighbors by tuning its update parameters. 
We show that the regret of our algorithm provably interpolates between the  $O(\sqrt{KT\ln K})$ regret for the classical bandit setting and the $O(\sqrt{T\ln K})$ regret for the full-information setting (where the full vector $\mathbf{g}(t)$ is observed at each time step) depending on the {amount} of exploration conducted by the neighbors (see {Theorem~\ref{thm:mabn_ub}}), 
and is optimal up to log factors (see {Theorem~\ref{thm:gen_lb} \& \ref{thm:arm_lb}}). 
Moreover, our algorithm improves performance over the state-of-the art in silico.
The theoretical results are presented in {Section~\ref{sec:adv}} and Appendix~\ref{sec:adversarial}), and the empirical results in {Section~\ref{sec:empirical_adv}}).

\section{Related Work}
\label{sec:related_work}

Distributed learning in a network is a broad topic and has been studied  under various names in several disciplines.
However, to the best of our knowledge, our model for learning from neighbors, along with its assumptions and non-assumptions (1)-(7) which are motivated by relevant settings in sociology and economics, is novel.
 Here we briefly survey the closest relatives to our work.

In the study of \emph{non-strategic learning on networks}, individuals are connected via a network, and each individual has a finite set of actions with probabilistic rewards whose distributions depend on the \emph{state of the world} (see \cite{G2005}, Chapter 2 for a survey). 
Indeed, would be similar to our model in the stochastic setting. 
However, work in this area has focused on studying variants of a \emph{greedy} algorithm, and answering the question of whether learning (i.e., discovery of the state of the world, and hence convergence to the best action) occurs asymptotically (see, e.g.,~\cite{BG1998,EF1993,BG2001,GK2003,GJ2010}). 
 Instead, we are concerned with \emph{regret}, which could be loosely interpreted as the \emph{rate of convergence}. 
 

Recall that in models of bandits with \emph{side observations}, in addition to observing $g_{a(t)}(t)$, one may \emph{observe} (but not receive) additional rewards $g_{S(t)}(t)$. The set of arms $S(t)$ depends on the particular model of side observations. 
In the free observations model, the individual can select $B$ additional arms to observe at each time step; i.e., $|S(t)| = B$ and the individual  selects $S(t)$ for all $t$. Such models have been studied both for stochastic (\cite{YM2009}) and adversarial (\cite{AMS2012, AKTT2015}) bandits. 
Without assumption (7), we could apply such algorithms directly because an individual could dictate which actions her neighbors should take. In the social setting we cannot hope to control our neighbor's decisions in this manner.
Still, we show that the performance of our algorithm is equivalent empirically to such algorithms (see Section~\ref{sec:empirical}).
%
In the arm-network (or action-network) setting, 
the individual observes the rewards of the neighbors of the {arm} she selects. 
Such stochastic (\cite{CKLB2012,BES2014}) and adversarial~(\cite{MS2011,ACGM2013,KNVM2014,ACDK2015}) bandit settings have been studied. 
While one could apply these algorithms to the social setting, some social information, in particular from assumption (3), is left on the table.
Leveraging this allows us to provably outperform such approaches (see {Section~\ref{sec:better}}), and empirically the difference can be dramatic (see {Figure~\ref{fig:main}}).

Other work has considered bounding the \emph{cumulative} regret of all individuals, rather than individuals minimizing their own regret. 
Towards this, centralized algorithms for various versions of stochastic bandits have been studied, in particular for the complete graph (\cite{BES2013, SBHOJK2013,CGZ2013}). 
Although the centralized setting is not the object of our study, as a corollary, we obtain a centralized algorithm for adversarial bandits that is optimal on the complete network (see {Section~\ref{sec:centralized}}).

\section{Technical Overview for the Stochastic Setting}
\label{sec:tech_stoch}

To describe our algorithm, let us first revisit the $\UCB$ algorithm, first introduced by Auer et. al. \cite{ACF2002} and since widely extended and studied~(see, e.g., \cite{B2010, MMS2011, GC2011}).  
$\UCB$ is an asymptotically optimal algorithm for stochastic bandits with many well-studied variants (see, e.g., \cite{BanditBook} for an overview). 
The idea behind the algorithm uses the principle of \emph{optimism in the face of uncertainty}; the algorithm maintains an optimistic \emph{upper bound} on the mean reward of each arm, and selects an arm with maximal upper bound. As is standard, we assume the probability distributions satisfy Hoeffding's lemma.
Then, arm $j$ at time $t$ has an upper bound 
\begin{equation} \label{rule}
U_j(t) = \hat\mu_j(t) + \sqrt{\frac{\alpha \ln(t)}{2n_j(t)}} 
\end{equation}
which holds with probability at least $1- t^{-\alpha}$ 
when $\hat\mu_j(t)$ is the sample mean of arm $j$, and $n_j(t)$ is the number of samples we have for $j$. 
At time $t$, an arm $a(t) \in \mbox{arg}\max_j \{U_j(t)\}$ %
is selected. 

We make a simple extension to $\UCB$ for an agent on a network: the agent simply incorporates \emph{all} samples and \emph{all} rewards into $n_j$ and $\hat\mu_j$ regardless of whether it came from her action or was observed from one of her neighbors. 
We denote this algorithm by  $\UCBN$, and 
note that it can be implemented by an individual irrespective of the graph structure and the algorithm(s) her neighbors may employ.   
The regret of  $\UCBN$ algorithm is upper bounded as follows. 
\begin{thm} \label{thm:UCBN_ub}
Consider an agent with neighbors who play arbitrarily.  Let $n_i^\prime(t)$ be the number times arm $i$ has been selected by one of her neighbors by time $t$. %
	Then, the regret  of $\UCBN$ for any  $\alpha >2$ is
	\begin{equation}\label{bound2}
		\begin{aligned}
			R \leq  \sum_{i,\Delta_i>0}\left(\max \left \{\max_{t=1,..,T} \left\{\frac{2\alpha\ln t}{\Delta_i}-n^\prime_i(t)\Delta_i\right\} , 0  \right\} + \frac{\alpha}{\alpha-2}\right), 
		\end{aligned}
	\end{equation}
	where $\Delta_i$ is the difference between $\mu_{i^\star}$ and  $\mu_i$.
\end{thm}
This result is asymptotically optimal (see Theorem~\ref{thm:UCBN_LB}). 
The regret differs from the regret of the classic $\UCB$ regret by the $-n_j^\prime(T)\Delta_j$ term, and, depending on the behavior of the neighbors, can potentially take the agent from logarithmic to constant regret.

Clearly, the performance of an agent must depend on the behavior of her neighbors. In the worst case, if there are \emph{clumsy} agents who always select the same arm, then our regret is not improved much. However, as long as the agent has at least one neighbor who explores an arm uniformly at random with probability $\eps_t \in \omega(\frac{K \ln t}{t})$ at time $t$ (e.g., this occurs if a neighbor uses an adaptive greedy algorithm), then the regret is $O(1)$! Hence, this allows us to interpret neighbor behavior to our regret seamlessly. 
As an instructive example, consider the setting where all agents use $\UCBN$ in a complete graph. The regret in this setting is $O\left(\frac{K\ln T}{b}\right)$. 
In other words, the regret of an agent using $\UCBN$ is a factor $O(\sfrac 1 b)$ less than that of an agent using $\UCB$ -- indeed we cannot hope to do better, even in a completely centralized setting.
The proof parallels the proofs for the original UCB results (see, e.g., \cite{BanditBook} for a template), and can be found along with further discussion in Appendix~\ref{sec:stochastic}.
While the story for the stochastic setting turns out to be simple and easy to manage, the adversarial setting, as we see below, turns out to be more challenging.

\section{Technical Overview for the Adversarial Setting}
\label{sec:adv}

\subsection{Preliminaries}

The multiplicative weight update method has been discovered many times in many fields over the past century (see ~\cite{AHK2012} for an overview). It is a simple yet surprisingly powerful way to \emph{conservatively} update beliefs about the benefit of a given arm and is extremely effective for adversarial bandits and is asymptotically optimal up to log factors (see, e.g.,~\cite{ACFS2002, FKM2005, ACGM2013}). 
Such algorithms for the full information setting (where all rewards are observed at each time step) maintain a vector of weights $w_j$ for each arm $j$, and (multiplicatively) update it at each time step: 
$
w_j(t+1)=w_j(t) e^{\delta   g_j(t)},  
$
where $0 \leq g_j(t) \leq 1$ is the \emph{reward} observed for arm $j$ at time $t$ and $\delta$ is the \emph{update parameter}.
The probability of choosing arm $j$ at time $t$ is proportional to the weight $w_j(t)$, namely,
$p_j(t)= \frac{w_j(t)}{W_t}$ 
where $w_j(0) = \sfrac 1 N$,  and $W_t = \sum_{j=1}^N\ w_j(t)$.
This algorithm, for an optimal choice of $\delta$ has regret $\Theta(\sqrt{T\ln K})$.
Extending to the bandit setting uses a simple trick; instead of using $g_j(t)$ to update, we use an \emph{unbiased estimator} $\hat g_j(t)$ for $g_j(t)$ (see, e.g., \cite{ACFS2002, FKM2005}).
Since we no longer observe the reward of all of the arms, we must also ensure some exploration should be added to the algorithm. This is achieved by setting a lower bound $\eta \in [0,1]$ (the exploration parameter) on the probability of exploration:  
$p_j(t)=(1-\eta) \frac{w_j(t)}{W_t}+ \frac{\eta}{K}.$
This algorithm, also known as $\EXPThree$ \cite{ACFS2002}, achieves regret $O(\sqrt{TK\ln K})$, and is optimal up to log factors for the right choice of parameters 
(see \cite{BanditBook} for an exposition). 

\subsection{Formal Statement of Results}

We call our algorithm  in the adversarial setting $\MABN$. 
Recall that $p_j(t)$ is the probability that an individual selects arm $j$ at time $t$. Let $q_j^i(t)$ be the probability that her neighbor $i$ selects arm $j$ at time $t$. We denote the number of an individual's neighbors by $b$. Note that the number of nodes in a network, denoted by $N$, may be much larger, but the remaining network does not play a role in the algorithm or main results.
\begin{thm}
\label{thm:mabn_ub}
Given an individual with $b$ neighbors who are playing arbitrarily, the regret when using the $\MABN$ algorithm is 
	\begin{equation}
	R_\EXPN \leq \mathop{\mathbb{E}}\left[2\sqrt{\left(T+\sum_{t=1}^{T}\gamma_t\right) \ln K}\right]
	\end{equation}
	where $\gamma_t=\sum_{j=1}^{K}\frac{p_j(t)}{p_j(t)+\sum_{\ell=1}^{b}q_j^\ell(t)}$.
\end{thm}
Before discussing the proof, we first re-write the results in a way that makes them easier to interpret.

\label{sec:interpolate}
{For ease of presentation, momentarily assume that for all arms $j \in [K]$, neighbors $i \in [b]$, and times $t \in [T]$ we have that $q_j^i(t) \geq \frac{\eps_i}{K}$ for some $\eps_i \in (0,1]$.\footnote{This assumption is not required for the proof of Theorem~\ref{thm:mabn_ub}, only for the ease of interpretation in Equation~\ref{regretbound_simple}. Note that if a neighbor is running any variant of the multiplicative weight update method, this condition is satisfied. Removing this assumption requires the \emph{number} of non-zero $\eps_j$ to be tracked for each $j$, and these numbers would appear in the regret bound. }
We can then reinterpret $R_\MABN$ as a function of the bandit regret ($R_{\EXPThree}$) and full information information regret ($R_{\MUA}$) as follows: 
	\begin{equation}\label{regretbound_simple} 
		{R_{\MABN}}= \left\{
		\begin{array}{ll}
			R_{\MUA}\cdot \sqrt{\beta} & \Theta \leq 1-\sqrt{\frac{\ln K}{\beta T}} \\
			
			R_{\EXPThree} \cdot \sqrt{\frac{\beta+K/\Theta}{2K}} 
			&  1-\sqrt{\frac{\ln K}{(\beta+ \frac K \Theta)T}} \leq \Theta 
					\end{array} \right.
	\end{equation} 
where
$\Theta = \Pi_{i=1}^b \left(1-\frac{\eps_i}{K}\right)  \mbox{ and } \beta = \frac{1}{1-\left(1-\frac{1}{K}\right)\Theta} + 1.$

\noindent In particular, note that when $\Theta = 1$, none of the individual's neighbors maintain a probability distribution that is bounded away from 0 for all arms. In other words, the neighbors are not exploring effectively. In this case, $\beta = K+1$ and  $R_{\MABN} \in O(\sqrt{TK\ln K})$, the same as in the classical bandit setting. On the other hand, for example, when $\Theta \leq \sfrac 1 2$, then $\beta \leq 3$ and hence $R_{\MABN} \in O(\sqrt{T \ln K})$, the same as in the full-information setting. Hence, this algorithm smoothly interpolates between bandit regret and full information regret as a function of the neighbors' exploration.

The proofs of Theorem~\ref{thm:mabn_ub} and Equation~\ref{regretbound_simple} {appear in {Section~\ref{sec:gammaubmain} and Appendix \ref{sec:adversarial}}} respectively. 
At first, the proofs parallel standard approaches to analyze the multiplicative-weight update method; the crucial difference is a new unbiased estimator that is used in order to incorporate the neighbors' information (see Section~\ref{sec:algorithm_adv}).  
This leads to the following bound on the regret:
\[R \leq \frac{\ln K}{\delta} + \eta   T +  \delta T \sum_{j=1}^K \frac{p_j(t)}{p_j'(t)}.\]
The technical obstacle then becomes attaining tight bounds on the $\sum_{j=1}^K \frac{p_j(t)}{p_j'(t)}$ term (see Lemmas~\ref{lem:UB} and \ref{lem:helper}). 
 What remains is then a straightforward (albeit tedious) optimization problem on the parameters $\eta$ and $\delta$.

We can further show that Theorem~\ref{thm:mabn_ub} is optimal up to log factors. 
\begin{thm}
\label{thm:gen_lb}
Given an individual with $b$ neighbors who are playing arbitrarily, the regret when using the $\MABN$ algorithm is 
\[R_{\MABN} = \Omega \left(\sqrt{T + \sum_{t=1}^T \gamma_t}\right)\]
where $\gamma_t=\sum_{j=1}^{K}\frac{p_j(t)}{p_j(t)+\sum_{\ell=1}^{b}q_j^\ell(t)}$ as defined above.
\end{thm}
This information theoretic lower bound follows by extending the lower bound attained for the classic bandit setting  (see, e.g., \cite{ACF2002}), and is given in {Appendix \ref{sec:lowerbound}}. 
This theorem shows that the analysis of our algorithm is tight up to log factors and will help us establish dominance over other potential approaches as discussed in Section~\ref{sec:better} -- a weaker lower bound that is not algorithm-dependent (which matches the above bound for pathological cases such as when all neighbors always play the same action) is presented in Theorem~\ref{thm:arm_lb}.

\subsection{The $\MABN$ Algorithm}
\label{sec:main}
\label{sec:algorithm_adv}

Key to our $\MABN$ algorithm is the following new unbiased estimator for the rewards:
\begin{equation} \label{eq:MABNestimator} 
	\hat{g}_j(t)=
	\begin{cases}
	\frac{g_j(t)}{p_j^\prime(t)} &\mbox{if some individuals selects action j at time t}\\
	0 &\mbox{otherwise,}
	\end{cases}
\end{equation}
where $p^\prime_j(t)$ is the probability that at least one  individual selects action $j$, i.e.,
\begin{equation} \label{pPrime}
p^\prime_j(t) \defeq 1- (1-p_j(t))(1-q^1_j(t))\cdots(1-q^b_j(t)).
\end{equation}
The algorithm then updates the weights according to $w_j(t+1)=w_j(t)  e^{\delta \hat{g}_j(t)} = w_j(0) e^{\delta \sum_{s=1}^t \hat{g}_j(s)},$
where $w_j(0) = 1$, 
and updates the probability distributions according to 
$ p_j(t)=(1-\eta)\frac{w_j(t)}{W_t}+\frac{\eta}{K} $
 where $W_t = \sum_j w_j(t)$. 
Note that this algorithm can be implemented irrespective of the network structure and depends only on the information obtained locally from neighbors as defined in our model.
In essence, the key to our algorithm is two fold:
\begin{enumerate}[topsep=0pt,itemsep=0pt,partopsep=0ex,parsep=0ex]
\item \emph{Design} a new unbiased estimator $\hat{g}_j(t)$ that incorporates the side observations obtained from neighbors: Unlike for stochastic bandits, na\"ive  estimators do not suffice, and a new approach is required.\footnote{We must ensure that in bounding $\mathop{\mathbb{E}}[\hat{g}^2_j]$, we get some improvement over the usual bandit setting; it is easy to verify that such bounds do not hold for na\"ive estimators such as the average of the neighbors' estimators.} %
\item \emph{Decouple} the exploration and exploitation parameters: When an individual's neighbors explore a lot, she could benefit by  \emph{free-riding} off of the exploration of her neighbors; this is accomplished by decreasing her exploration parameter $\eta$. However, if we take $\delta = \eta/K$ as in $\EXPThree$, this dampens our updates. Hence we need $\eta$ and $\delta$ to act independently. %
\end{enumerate}

\noindent This analysis is presented in {Appendix \ref{sec:oldub}.

\subsection{Proof of Theorem~\ref{thm:mabn_ub}}
\label{sec:gammaubmain}

While the above version of the algorithm gives a natural interpretation of the parameters, in order to attain the stronger regret bound in Theorem~\ref{thm:mabn_ub} we take a slightly different approach. Instead of decoupling the parameters $\eta$ and $\delta$, we instead use an adaptive $\delta_t$ that changes based on the amount of information received from the neighbors.\footnote{Adaptive $\delta_t$ are often used when $T$ is unknown; here the adaptive $\delta_t$ serves a different function by allowing us to explicitly respond to the neighbors' actions.}   In particular, we let $\delta_t=\sqrt{\frac{\ln K}{\sum_{c=1}^{t}(1+\gamma_c)}}$ and can take $\eta = 0$. 
Importantly, note that we use the same unbiased estimator in either version of the algorithm.

\begin{proof}[Proof of Theorem~\ref{thm:mabn_ub}]
The first part of proof (from Equation~\eqref{eqmain:eq1} to Equation~\eqref{eqmain:main}) parallels the traditional analysis for analyzing multiplicative weight update algorithms; for completeness we present the steps without going into the details (see \cite{BanditBook} for an exposition). We first write $\mathop{\mathbb{E}}[\hat{g}_j(t)]$ as follows:
	\begin{equation}\label{eqmain:eq1}
	\mathop{\mathbb{E}}[\hat{g}_j(t)]=\frac{1}{\delta_t}\left(\ln \mathop{\mathbb{E}}\left[\exp\left(-\delta_t\left(\hat{g}_j(t)-\mathop{\mathbb{E}}[\hat{g}_j(t)]\right)\right)\right]
	-\ln \mathop{\mathbb{E}}\left[\exp\left(-\delta_t\hat{g}_j(t)\right)\right]\right)
	\end{equation}
	where the expectation is over the randomness of the estimator and choice of the arm.
	We will now consider the right-hand side of the equation and upper bound the two terms separately. 
	\begin{equation}
	\begin{aligned}
	\frac{1}{\delta_t}\ln \mathop{\mathbb{E}}\left[\exp\left(-\delta_t(\hat{g}_j(t)-\mathop{\mathbb{E}}[\hat{g}_j(t)])\right)\right] &=
	\frac{1}{\delta_t}\ln \mathop{\mathbb{E}}\left[\exp(-\delta_t\hat{g}_j(t)\right]+\mathop{\mathbb{E}}[\hat{g}_j(t)]) \\&\leq
	\frac{1}{\delta_t}\mathop{\mathbb{E}}[\exp(-\delta_t\hat{g}_j(t))-1+\delta_t\hat{g}_j(t)] 
	\leq
	\frac{\delta_t}{2}\mathop{\mathbb{E}}[\hat{g}^2_j(t)],
	\end{aligned}
	\end{equation}
	where we use the inequalities $\ln x\leq x-1$ and $\exp(-x)-1+x\leq x^2/2$ for $x\geq 0$.
	Now, let $\hat{G}_j(t)=\sum_{a=1}^{t}\hat{g}_j(t)$, and let $\psi(t)=\frac{1}{\delta_t}\ln\left(\frac{1}{K}\sum_{j=1}^{K}\exp(-\delta_t \hat{G}_j(t))\right)$. 
	Hence, 
	\begin{equation}\label{eqmain:eq2}
	\begin{aligned}
	-\frac{1}{\delta_t} &\ln \mathop{\mathbb{E}}_{a_t\sim p^\prime_t}\mathop{\mathbb{E}}_{j\sim p^\prime_t} \exp(-\delta_t \hat{g}_j(t))  \leq
	-\frac{1}{\delta_t} \mathop{\mathbb{E}}_{a_t\sim p^\prime_t} \ln \mathop{\mathbb{E}}_{j\sim p^\prime_t} \exp(-\delta_t \hat{g}_j(t)) \\ & 
	-\frac{1}{\delta_t} \mathop{\mathbb{E}}_{a_t\sim p^\prime_t}\left[ \ln\left( \sum_{j=1}^{K}\frac{\exp(-\delta_t \hat{G}_j(t))}{\sum_{c=1}^{K}\exp(-\delta_t \hat{G}_c(t-1))}\right)\right] 
	\leq \psi(t-1)-\psi(t)
	\end{aligned}
	\end{equation} 
	where we recall that $\eta = 0$, the first 
	inequality we use Jensen's inequality, 
	 and note that $\delta_t$ is decreasing in $t$.
	By summing up the terms in Equation~\eqref{eqmain:eq1} and ~\eqref{eqmain:eq2} over all $t$ we get
	\begin{equation}
	\sum_{t=1}^{T}\mathop{\mathbb{E}}[\hat{g}_j(t)] \leq \sum_{t=1}^{T}\frac{\delta_t}{2} \mathop{\mathbb{E}}[\hat{g}^2_j(t)]
	-\mathop{\mathbb{E}}_{a_t\sim p^\prime_t}\psi(T).
	\end{equation}
	
\noindent	We now bound $-\psi(T)$ as follows: 
	\begin{equation}
	\begin{aligned}
	-\psi(T) &=  \frac{\ln K}{\delta_T} - \frac{1}{\delta_T} \ln \left(\sum_{j=1}^{K}\exp(-\delta_T \hat{G}_j(T))\right) 
	\leq
	\frac{\ln K}{\delta_T} - \frac{1}{\delta_T} \ln \left(\exp(-\delta_T \hat{G}_k(T))\right)  
	=
	\frac{\ln K}{\delta_T} + \hat{G}_k(T).
	\end{aligned}
	\end{equation}
	Plugging this into in the above inequality yields
	\begin{equation} \label{eqmain:main}
	\sum_{t=1}^{T}\mathop{\mathbb{E}}[\hat{g}_j(t)] \leq \sum_{t=1}^{T}\frac{\delta_t}{2} \mathop{\mathbb{E}}[\hat{g}^2_j(t)]
	+\frac{\ln K}{\delta_T}
	+\mathop{\mathbb{E}}_{a_t\sim p^\prime_t} [\hat{G}_k(T)].
	\end{equation}

\noindent Now, we note that given $p_j(t)$ at time $t$ we have %
	\begin{subequations} \label{Expectation}
		\begin{align}
		\E\left[\hat{g_j}(t)\right]&=g_j(t), \\ 
		\E\left[\sum_{t=1}^{T}\sum_{i=1}^{K} p_j(t) \hat{g}_j(t)\right] 
		&= \sum_{t=1}^{T}\sum_{i=1}^{K} p_j(t)g_j(t)=\E\left[\sum_{t=1}^{T} g_{a(t)}(t)\right], \mbox{ and }\\ 
		\E\left[\sum_{t=1}^{T}\sum_{i=1}^{K}p_j(t)\hat{g}^2_j(t)\right] 
		&= \sum_{t=1}^{T}\sum_{i=1}^{K}\frac{p_j(t)}{p^\prime_j(t)}  g^2_j(t) 
		\leq  \sum_{t=1}^{T}\sum_{i=1}^{K}\frac{p_j(t)}{p^\prime_j(t)}   \label{ex3}
		\end{align}
	\end{subequations}
	as  $g_j(t)\leq 1$, and where  the expectation is over randomness of the algorithm. %
Now, attaining a good bound on the regret boils down to attaining a good bound on $\sum_{i=1}^{K}\frac{p_j(t)}{p^\prime_j(t)}$. Towards this, we need a technical lemma that, in effect, allows bounds the amount of information received from neighbors.
 \begin{lem}
 \label{lem:helper}
 	$\sum_{j=1}^{K} \frac{p_j}{1-(1-p_j)(1-q^1_j)\cdots(1-q^b_j)} \leq \sum_{j=1}^{K}\frac{p_j}{p_j+q^1_j+q^2_j+\cdots+q^b_j}+1$.
 \end{lem}
The proof is presented in Appendix \ref{sec:gammaubmain}. 
Using this Lemma and combining all of the above, we get  
	\begin{equation}\label{eqmain:regret_bound}
	R \leq \frac{\ln K}{\delta_T} +   \frac{1}{2}  \sum_{t=1}^{T}\delta_t(1+\gamma_t).
	\end{equation}
	To conclude the proof, recall that $\delta_t=\sqrt{\frac{\ln K}{\sum_{c=1}^{t}(1+\gamma_c)}}$, use Lemma 3.5 of \cite{auer2002adaptive}, and take expectation of the both sides of Equation~\eqref{eqmain:regret_bound}. 
\end{proof}

\subsection{Comparison to Alternate Approaches} 
\label{sec:feedback_graph}
\label{sec:better}

\begin{wrapfigure}{r}{.4\textwidth} 
  \begin{center}
  \vspace{-.65in}
	\includegraphics[width=2.75in]{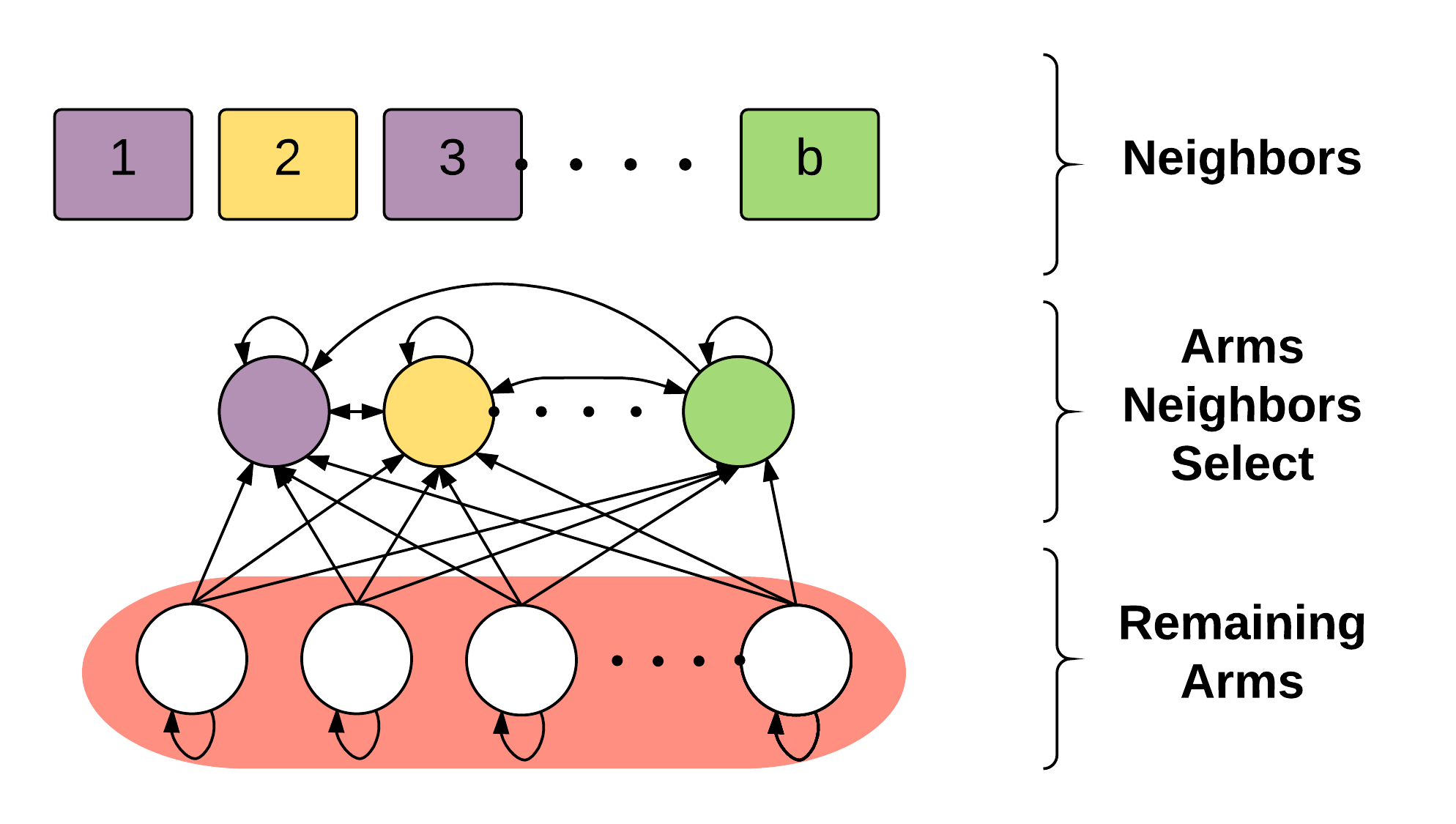}
	\caption{\small{The arm-network has an edge from arm $u$ to arm $v$ if, having selected arm $u$, we observe the reward of arm $v$. Arms selected by neighbors in the social network are form a clique, and the remaining arms have self loops and edges to all selected arms. }%
	}
		\label{fig:arm_network}
		  \end{center}
		    \vspace{-.25in}
\end{wrapfigure}

Instead of developing a new algorithm, we could have attempted to leverage an existing one. The most natural one to try is from the arm-network setting which is as follows: there is a single individual and the bandit's arms form an arm-network which can change over time. An edge from arm $u$ to arm $v$ means that by choosing arm $u$ we observe the reward of arm $v$.
Thus we could, in retrospect at each time step, recreate an arm-network (see Figure~\ref{fig:arm_network}) and apply an arm-network algorithm. We consider $\EXPG$ (\cite{ACDK2015}), which is the state-of-the-art solution for such problems, and performed best amongst arm-network algorithms in our empirical simulations. 
However, we prove in {Appendix~\ref{sec:comparison}} that our algorithm is at least as good.
\begin{prop} \label{thm:better}
$R_{\MABN} = O(R_{\EXPG})$.
\end{prop} 
Moreover, as we will see in Section~\ref{sec:empirical}, the regret of $\MABN$ is often drastically better empirically.  
This because $\EXPG$, and other similar algorithms, were developed from different settings in which it was not possible to make use of the probability distributions afforded to us by assumption (3). Indeed, without this assumption, one can get a stronger lower bound than the one presented above.
\begin{prop}
\label{thm:arm_lb}
Let $n_t$ be the size of the set of arms selected (arbitrarily) by all of the individual's neighbors at time $t$. Then, the regret $R_{\mathcal A}$ for any algorithm $\mathcal A$ in our setting without assumption (3) is 
$
R_{\mathcal A} = \Omega \left(\sqrt{T + \sum_{t=1}^T (K-n_t)}\right).
$
\end{prop}
The proof follows from Theorem~5 of \cite{ACGMMS2014}. %
Our $\EXPN$ algorithm is often able to beat this bound by leveraging (3). {For example, this proposition implies that if we have a complete network on $b$ vertices where $\log K \ll b \ll K$, then $R_\EXPG = \Omega\left(\sqrt{(K-b)T}\right)$ while in our case $R_\EXPN = O\left(\sqrt{\frac{K}{b}T}\right)$ (see Corollary \ref{cor:complete}).}

\subsection{A Centralized Solution for the Network}
\label{sec:centralized}

Our model and algorithm are formulated for an individual because this allows us to draw the most general conclusions -- bounding the individual's regret as a \emph{function} of the neighbors' behavior. 
However, a surprising feature is that 
it can also be made into a centralized solution.  
In the general case, this requires assuming there is an external coordinator that can select a maximum-degree individual to lead and direct the rest on how to act as follows: 
Let $v^\star$ be the maximum degree node selected. The coordinator directs $v^\star$ to use the $\EXPN$ algorithm. The remaining nodes $u$ are each assigned a neighbor $v_u$ that lies on the shortest path between them and $v^\star$, and are directed to copy the probability distribution that $v_u$ used in the previous time step. %
\begin{thm}
\label{thm:bmax}
Using  the above centralized algorithm, the regret of \emph{all} individuals is at most
	\begin{equation}
	\label{eq:bmax}
	R = O\left(\Delta + \sqrt{\left(1+\frac{K}{1+b_{max}}\right)T \ln K}\right),
	\end{equation}
	where $b_{max}$ is the degree of $v^\star$ and $\Delta$ is the diameter of the network. 
\end{thm}
The proof follows, with minor modifications, from the proof of Theorem~\ref{thm:mabn_ub}; the main difference regards accounting for the delay (of at most $\Delta$ time steps) for the farthest node from $v^\star$ to update their probability distribution. 
By replacing $\gamma_t$ with $\frac{K}{1+b_{max}}$, this gives us the resulting regret bound.  
In the simple case of a complete network on $N$ nodes, no coordinator is required, and we obtain the following corollary. 
\begin{cor}
\label{cor:complete}
On a complete network with $b$ nodes, if all nodes use the $\EXPN$ algorithm, then they  attain 
$R = {O}\left(\sqrt{\left(1+\frac{K}{b}\right)T\ln K}\right)$, 
which is optimal (up to log factors) for any centralized solution.
\end{cor}
This again follows from the proof of Theorem~\ref{thm:mabn_ub} using the fact that the number of neighbors is $b-1$ on a complete network, and that {a centralized solution has average regret $\Omega(\sqrt{\left(1+\frac{K}{b}\right)T})$ as shown in \cite{AKTT2015}}.

\begin{figure}
	\vspace{-.5in}
	\hspace{-.6in}
	\subfigure[Regret in a complete network on 5 nodes. We vary $T$ for  $K = 50$.\label{fig:regret_complete_T}]{\includegraphics[width=3.4in]{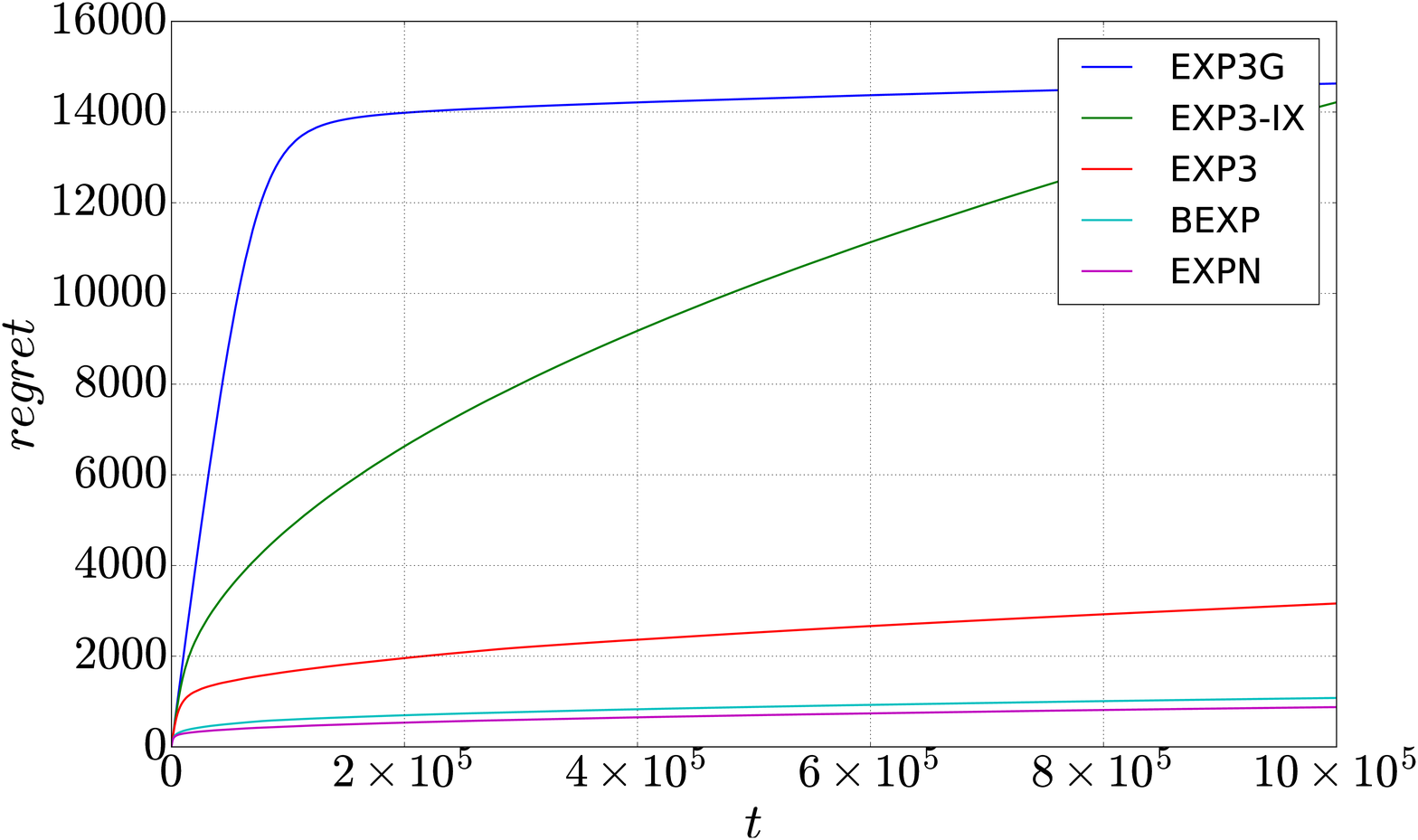}} \hfill
	\hspace{0.1in}
	\subfigure[Regret in a random 5-regular network with 50 nodes. We vary $T$ for $K = 50$.\label{fig:regret_breg}]{\includegraphics[width=3.4in]{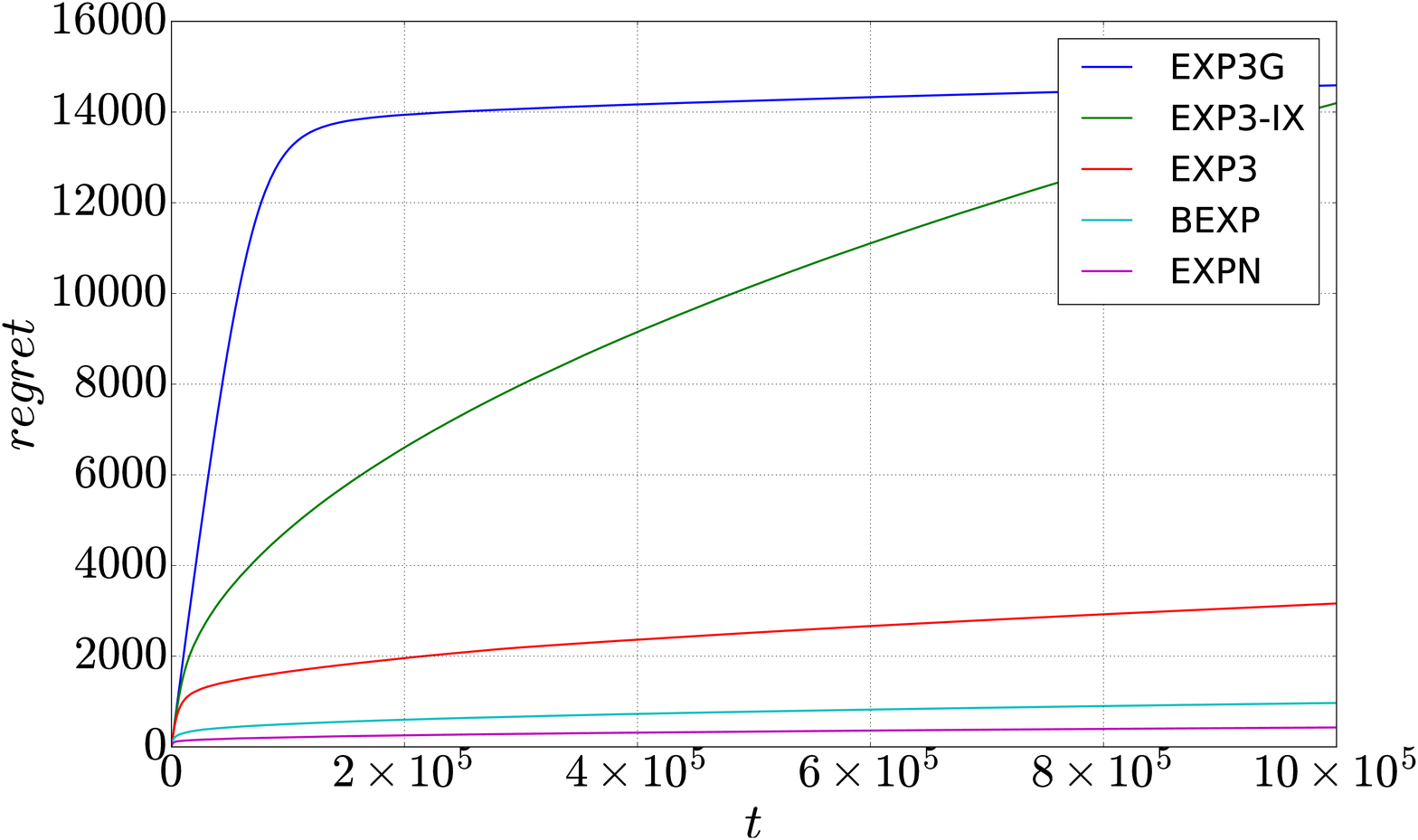}} 
	%
	
	\hspace{-.6in}
	\subfigure[The average regret in various network topologies. The size of the networks is 10.\label{fig:topology}]{\includegraphics[width=3.4in]{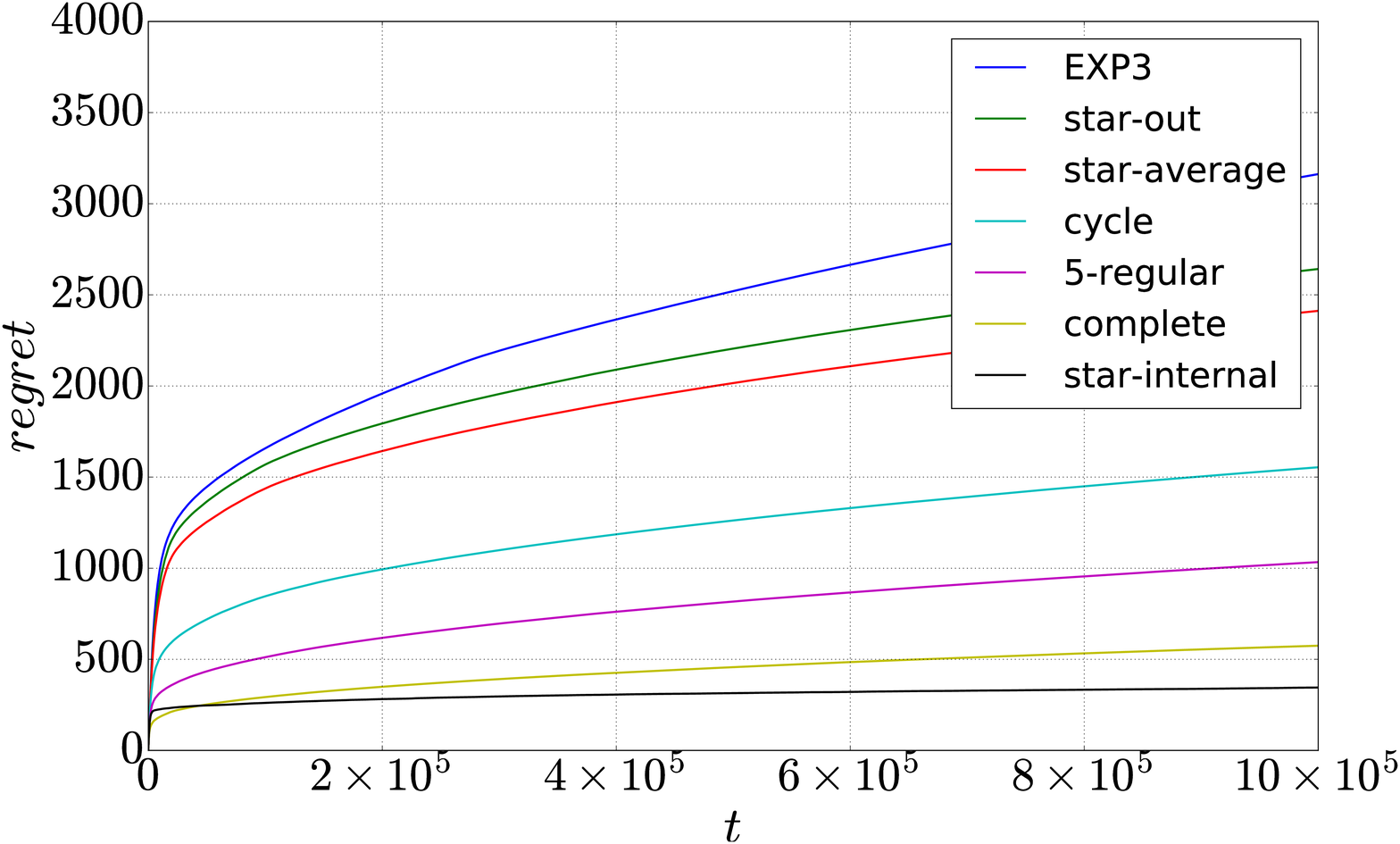}} \hfill
	\hspace{0.in}
	\hspace{0.1in}
	\subfigure[Regret ratio. We vary time $T$ for $K = 50$ and $N = 5$.\label{fig:ratio_complete_T}]{\includegraphics[width=3.4in]{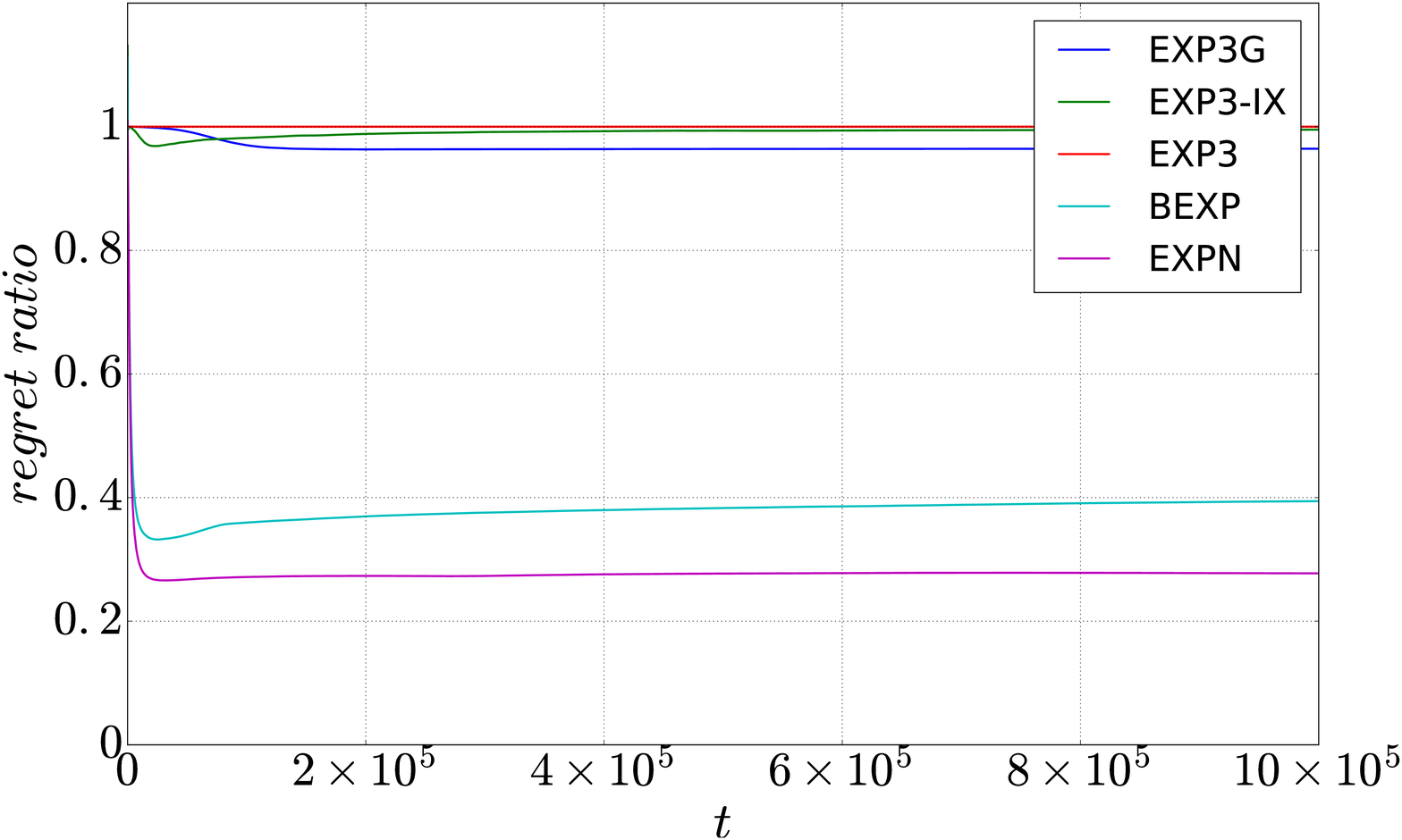}} 

	\hspace{-.6in}
	\subfigure[Regret ratio. We vary the number of arms $K$ for $T = 5 \cdot 10^4$ and $N = 5$.\label{fig:ratio_complete_K}]{\includegraphics[width=3.4in]{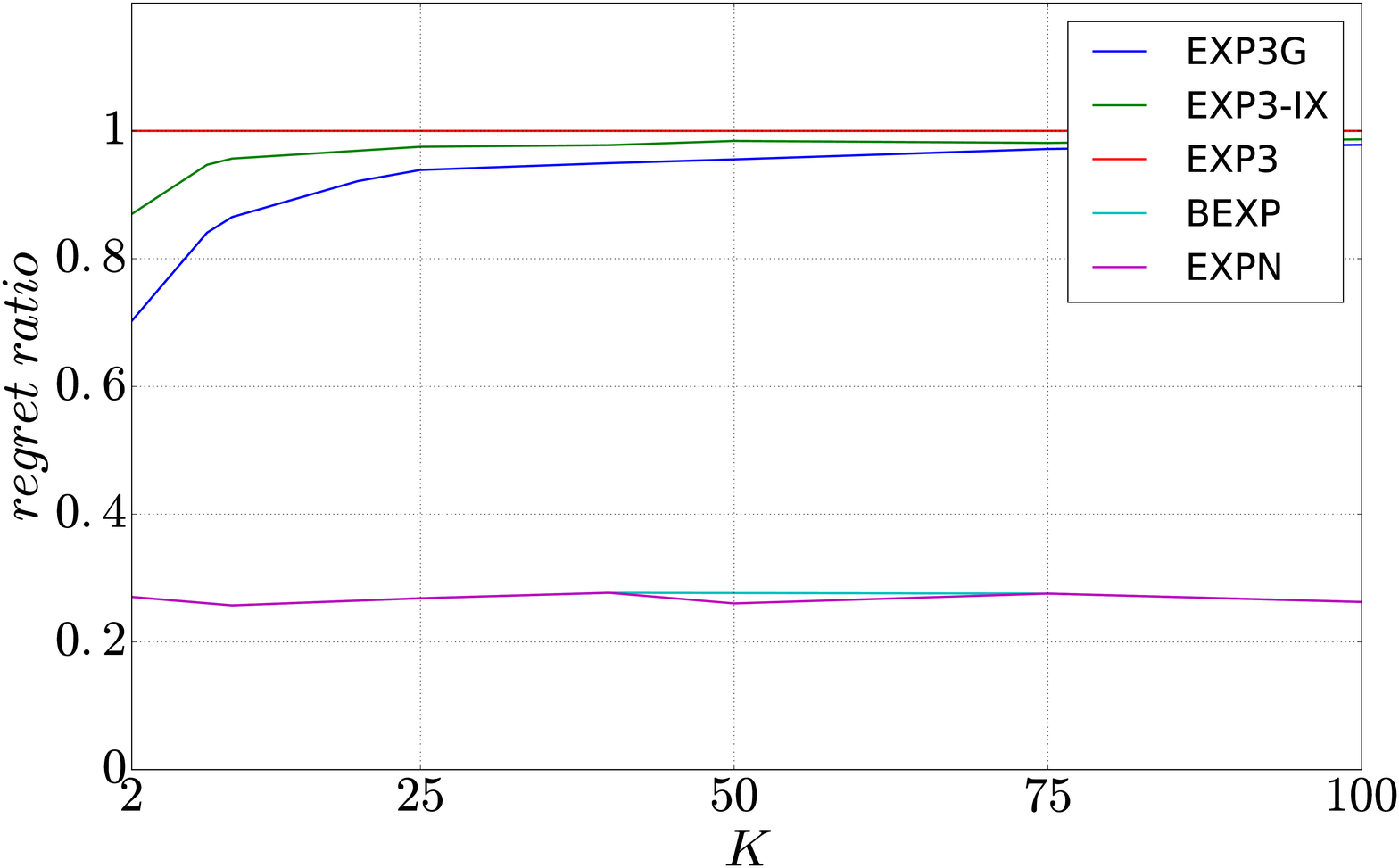}} \hfill
	\hspace{0.1in}
	\subfigure[Regret ratio. We vary the number of nodes $N$ for $ K = 50$ and $T = 10^6$. 
	 \label{fig:ratio_complete_b}]{\includegraphics[width=3.4in]{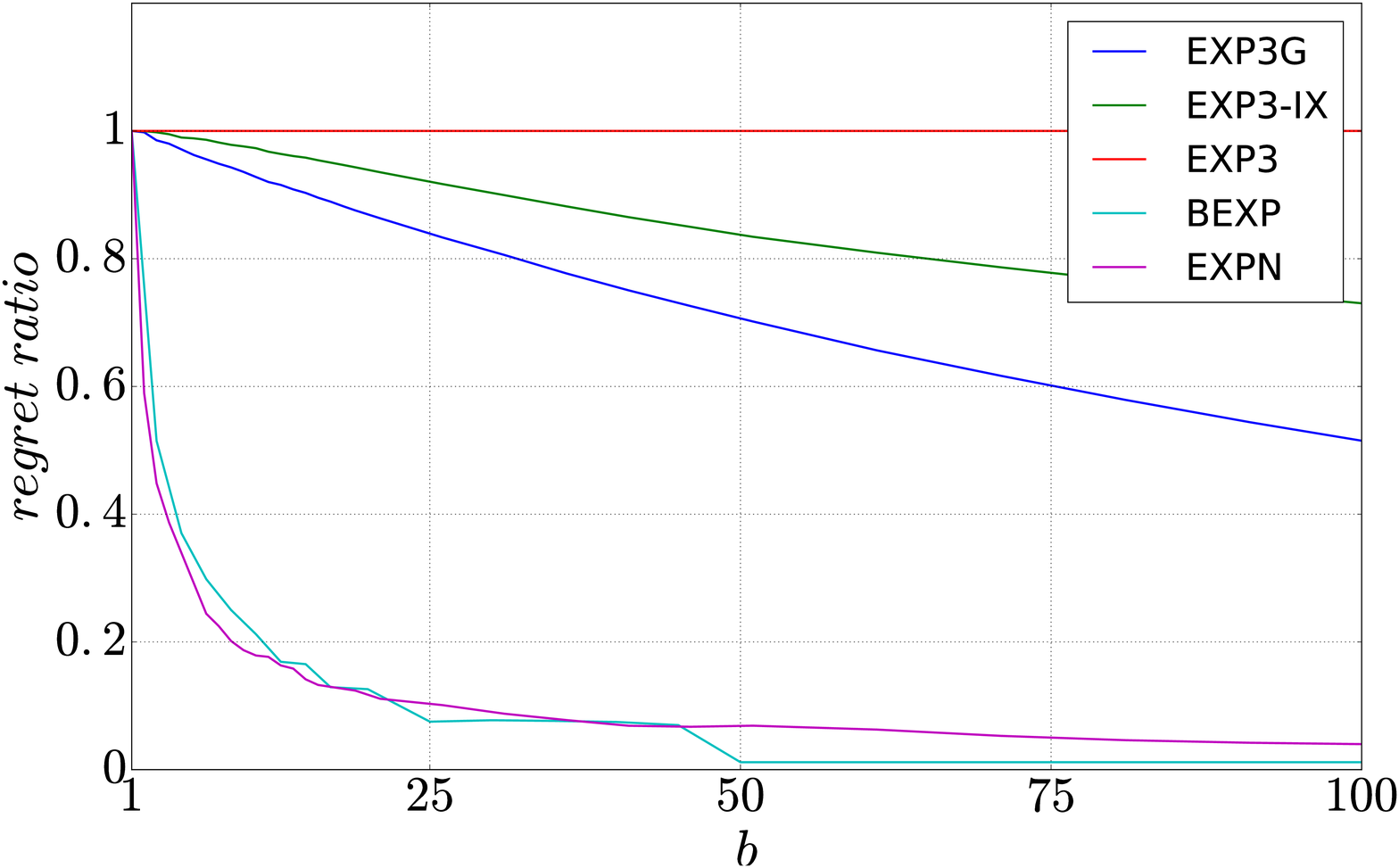}} \hfill
\vspace{-.05in}
\captionsetup{width=\textwidth}
\caption{{Performance of our algorithm ($\EXPN$) for the adversarial setting against benchmarks. Our algorithm significantly outperforms $\EXPThree$, indicating that the presence of neighbors indeed improves learning. It also significantly outperforms arm-network algorithms ($\EXPG$ and $\EXPIX$) that could be applied in our setting. Surprisingly, it's performance is as good as $\BEXP$, which would require a single node to dictate the choices of her neighbors; hence, our distributed algorithm is performing as well as a centralized one. Figures (a)-(c) depict the regret. Figures (d)-(f) depict the regret ratio, i.e., the ratio between an algorithm's regret with $b$ neighbors over its regret with 0 neighbors (where $b$ depends on the network structure addressed in the corresponding subfigure).}}
\label{fig:main}
\end{figure}

\section{Empirical Results}
\label{sec:empirical}

\subsection{Adversarial Setting}
\label{sec:empirical_adv}

\paragraph{\textbf{Benchmarks.}}
We compare our algorithm against the bandit algorithms developed for various settings with side-information, namely $\EXPG$ (\cite{ACDK2015}), $\EXPIX$ (\cite{KNVM2014}) and $\BEXP$ (\cite{AKTT2015}). 
The first two are designed for the arm-network setting as described in {Section~\ref{sec:better}}, while the latter is designed for the free-exploration setting described in Section~\ref{sec:related_work}.
Recall that in free-exploration there are no neighbors; rather there is a budget $B$, and at each time step the individual can choose up to $B$ arms to select. 
In order to attain a fair comparison, we assume we have budget $B = b+1$ for $\BEXP$, where $b$ is the number of neighbors.

\paragraph{\textbf{Experimental Setup.}}
For the simulations we use the decoupled version of the $\MABN$ algorithm as presented in Section~\ref{sec:algorithm_adv}; the results for the adaptive algorithm version would be even better. 
We consider a bandit with Bernoulli rewards that has a single \emph{good} arm with mean $0.7$, while the remaining arms have mean $0.5$. This is similar to the worst case (minimax) bandit; the difficulty arises from the fact that it is hard, in an information theoretic sense, to distinguish the single good arm from the rest with few samples. Indeed the performance for our algorithm in comparison to our benchmarks is only improved for all other settings we attempted.

\paragraph{\textbf{Performance in Networks.}}
In addition to exploring the effect of the various algorithm on a single individual, we are able to consider various network topologies and consider the regret as a whole. Towards this, in the first set of simulations, all nodes in the specified networks use the same algorithm. 
We first compare the regret of $\EXPN$ against the benchmarks in a complete network on $5$ nodes (Figure~\ref{fig:regret_complete_T}); even on such a small network the difference in regret is dramatic.\footnote{Indeed, on larger networks the differences are only more pronounced -- we present the results on a small network in order to be able to visualize them adequately.} We significantly outperform arm-network algorithms ($\EXPG$ and $\EXPIX$), which empirically are initially worse than even $\EXPThree$. Asymptotically $\EXPG$ eventually outperforms $\EXPThree$, although $\EXPIX$ does not. Surprisingly, our algorithm performs as well as $\BEXP$, which would be equivalent to identifying a single node as the leader and having them dictate the action of all other nodes. Hence, our distributed algorithm is as good as a centralized one. For comparison, we also consider a random $5$-regular graph on $50$ nodes (Figure~\ref{fig:regret_breg}), and observe that the performance of all algorithms is roughly equivalent to the complete network on $5$ nodes; i.e., the primary determining factor in the regret appears to be the number of neighbors rather than the topology of the network.

We also consider the regret of $\EXPN$ on various network topologies on 10 vertices: the complete network, a random {5}-regular network, a cycle, and a star network {(Figure~\ref{fig:topology})}. When the number of neighbors differ in a topology, the regret of the nodes may differ; the star is the extreme example and we depict the minimum (for the center node), maximum (for one of the leaves) and average regret. As expected, the more neighbors one has, the better the regret is, with the internal node of star outperforming all. We also observe that there is an advantage to  having neighbors that are not well-connected; despite a node in the complete network having the same degree as the center node of the star, the former has more regret. Because the nodes that are not well-connected receive less information, they must  explore more -- this is advantageous for their neighbors.

\paragraph{\textbf{Performance of Individuals.}}
Moving back to analyzing the performance for an individual, consider a setting where her neighbors all use the $\EXPThree$ algorithm.  
We measure the \emph{regret ratio}, i.e., the ratio between the regret of bandit algorithm $\mathcal A$ when the node has $b$ neighbors divided by the regret of $\mathcal A$ when the node has 0 neighbors. This allows us to better visualize the improvement in regret that each algorithm obtains as a function of the number of neighbors. We vary time $T$ (Figure  \ref{fig:ratio_complete_T}), the number of arms $K$ (Figure \ref{fig:ratio_complete_K} and the number of neighbors $b$ (Figure \ref{fig:ratio_complete_b}). We observe that, in all cases, our $\EXPN$ algorithm always matches or outperforms the benchmarks.  
The fact that the performance of our $\EXPN$ is comparable to that of $\BEXP$ is surprising, as we could not hope to do any better.

\subsection{Stochastic Setting}
\label{sec:empirical_stoch}

\begin{figure}
		\vspace{-.15in}
		\hspace{-.6in}
		\subfigure[Regret in a complete network on 5 agents. We vary $T$ for  $K = 5$.\label{fig:UCBregret_complete_T}]{\includegraphics[width= 3.4in ]{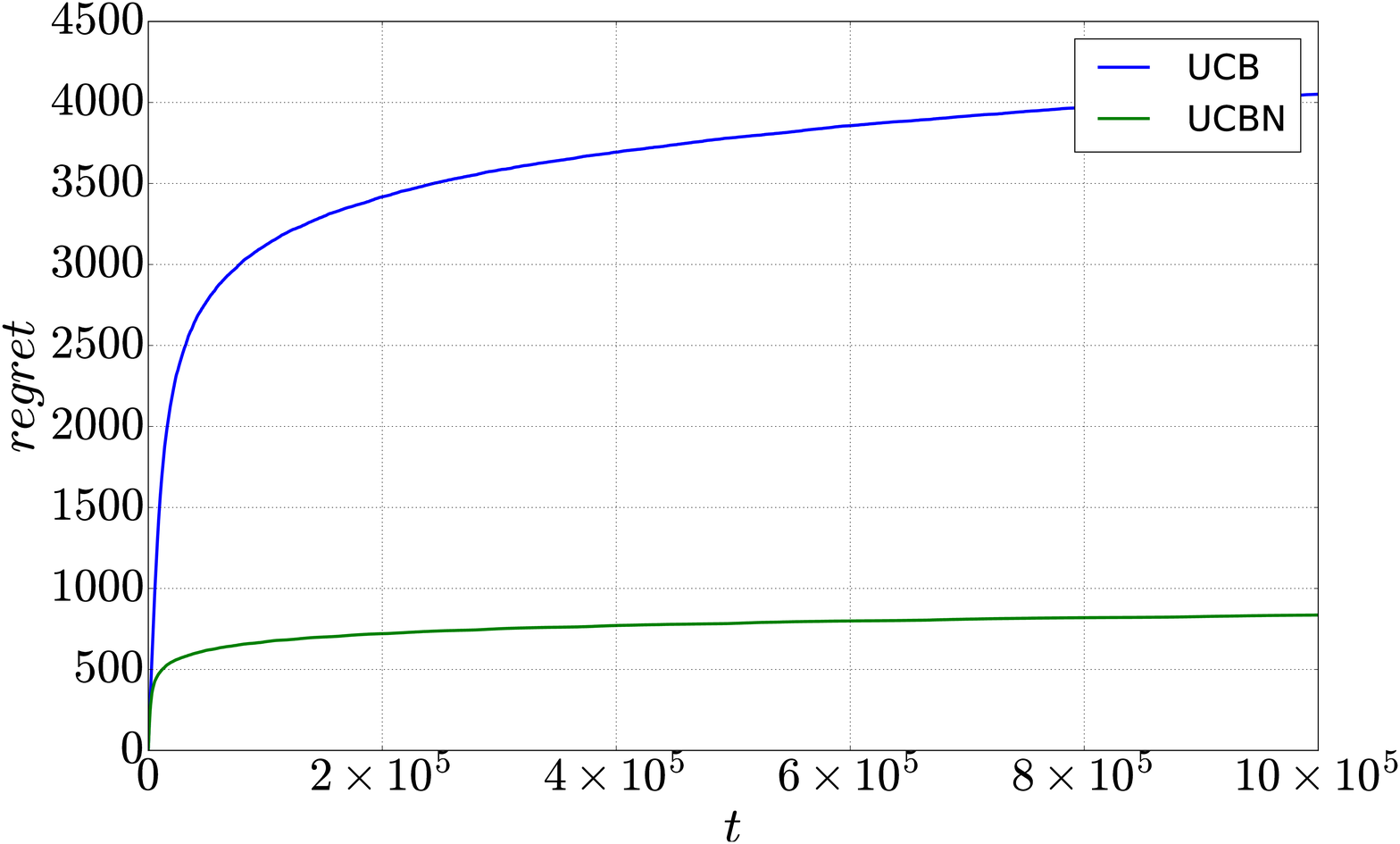}} \hfill
		\hspace{0.1in}
		\subfigure[Regret in a random 5-regular network with 10 agents. We vary $T$ for $K = 5$.\label{fig:UCBregret_breg}]{\includegraphics[width= 3.4in ]{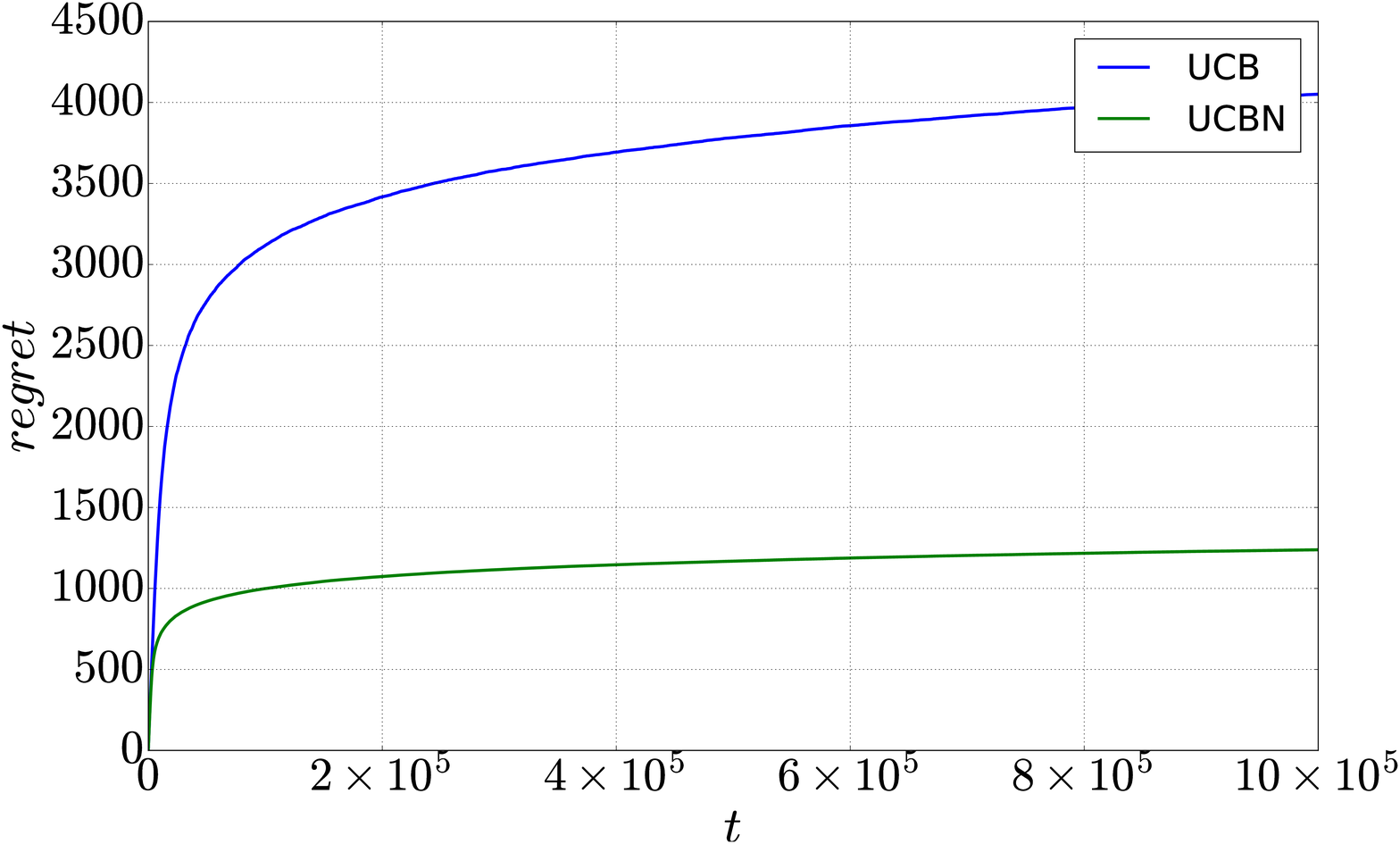}} \hfill

				\vspace{-.15in}
		\hspace{-.6in}
		\subfigure[The average regret in various network topologies. The size of the networks is 10, $K=50$ and all agents use $\UCBN$.\label{fig:UCBtopology}]{\includegraphics[width= 3.4in ]{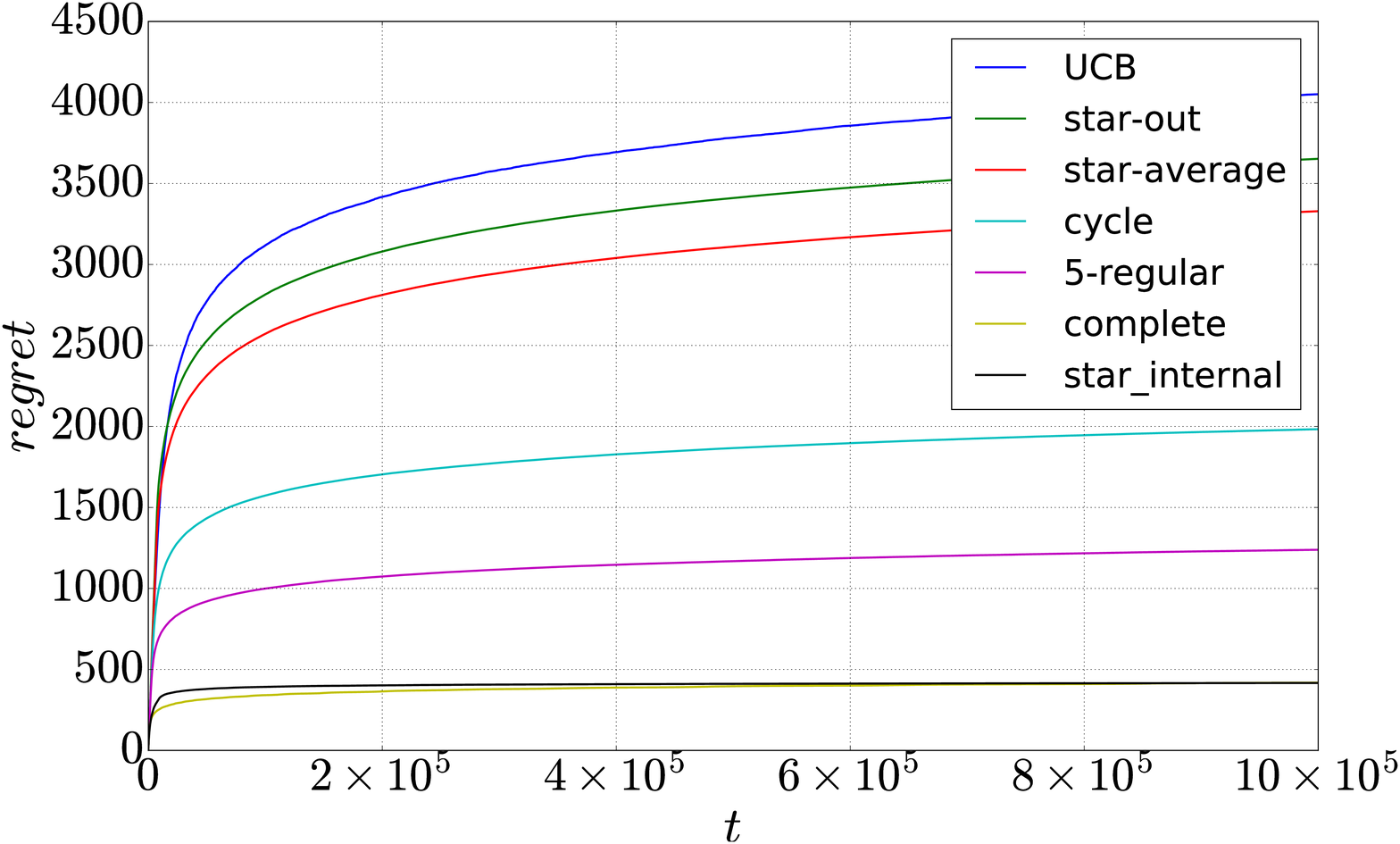}} \hfill		
		\hspace{0.1in}
	\subfigure[Regret in the complete network when we vary the number of agents for $ K = 5$ and $T = 25\cdot10^4$ .\label{fig:UCBcompare_GOB}]{\includegraphics[width= 3.4in ]{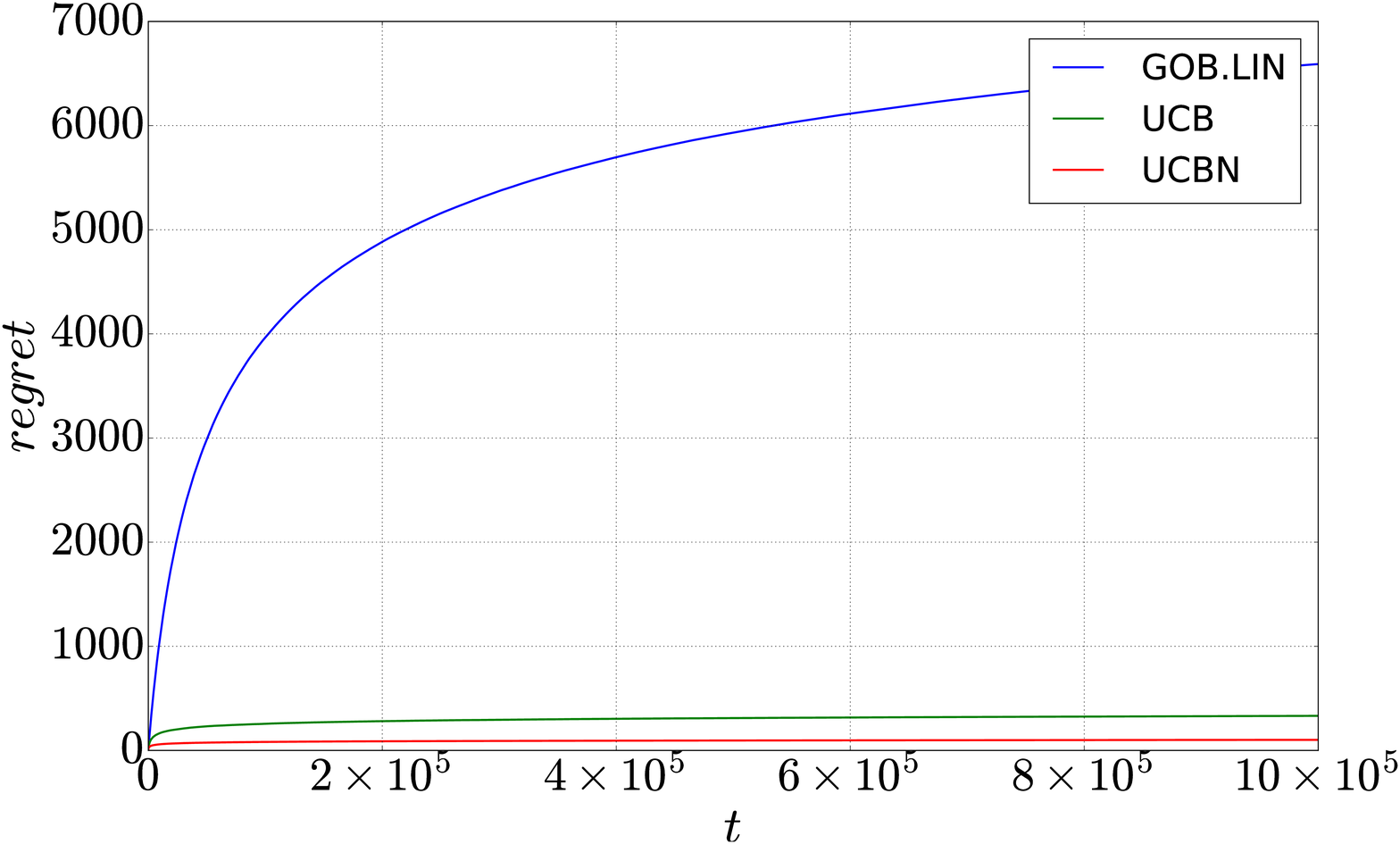}} \hfill

		\vspace{-.15in}
	\hspace{-.6in}
	\subfigure[Regret ratio in the complete network on 5 agents. We vary $K$ for $T = 25 \cdot 10^4$.\label{fig:UCBratio_complete_K}]{\includegraphics[width= 3.4in ]{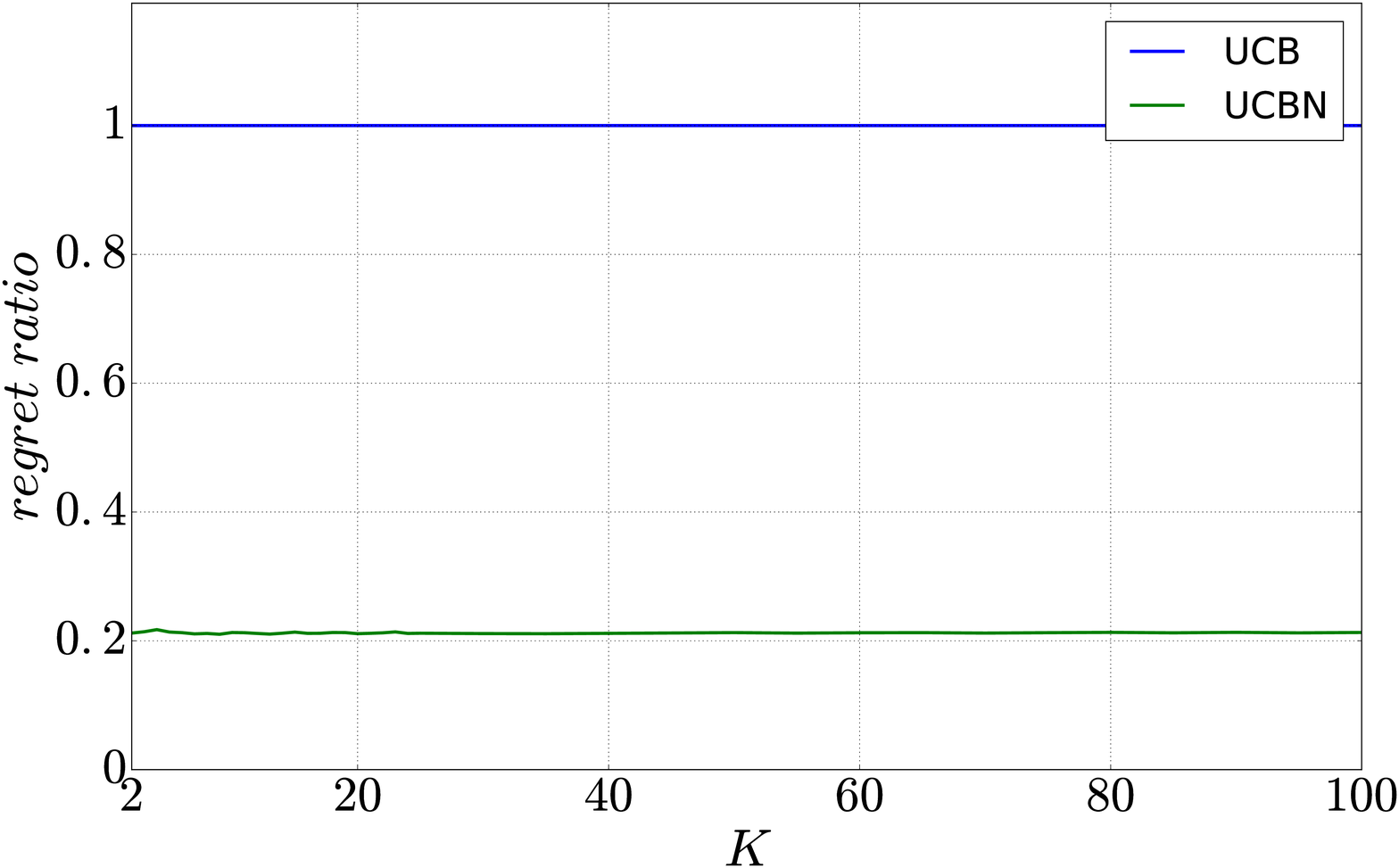}} \hfill
	\hspace{0.1in}
	\subfigure[Regret ratio in the complete network when we vary the number of agents for $ K = 50$ and $T = 25\cdot10^4$ .\label{fig:UCBratio_complete_b}]{\includegraphics[width= 3.4in ]{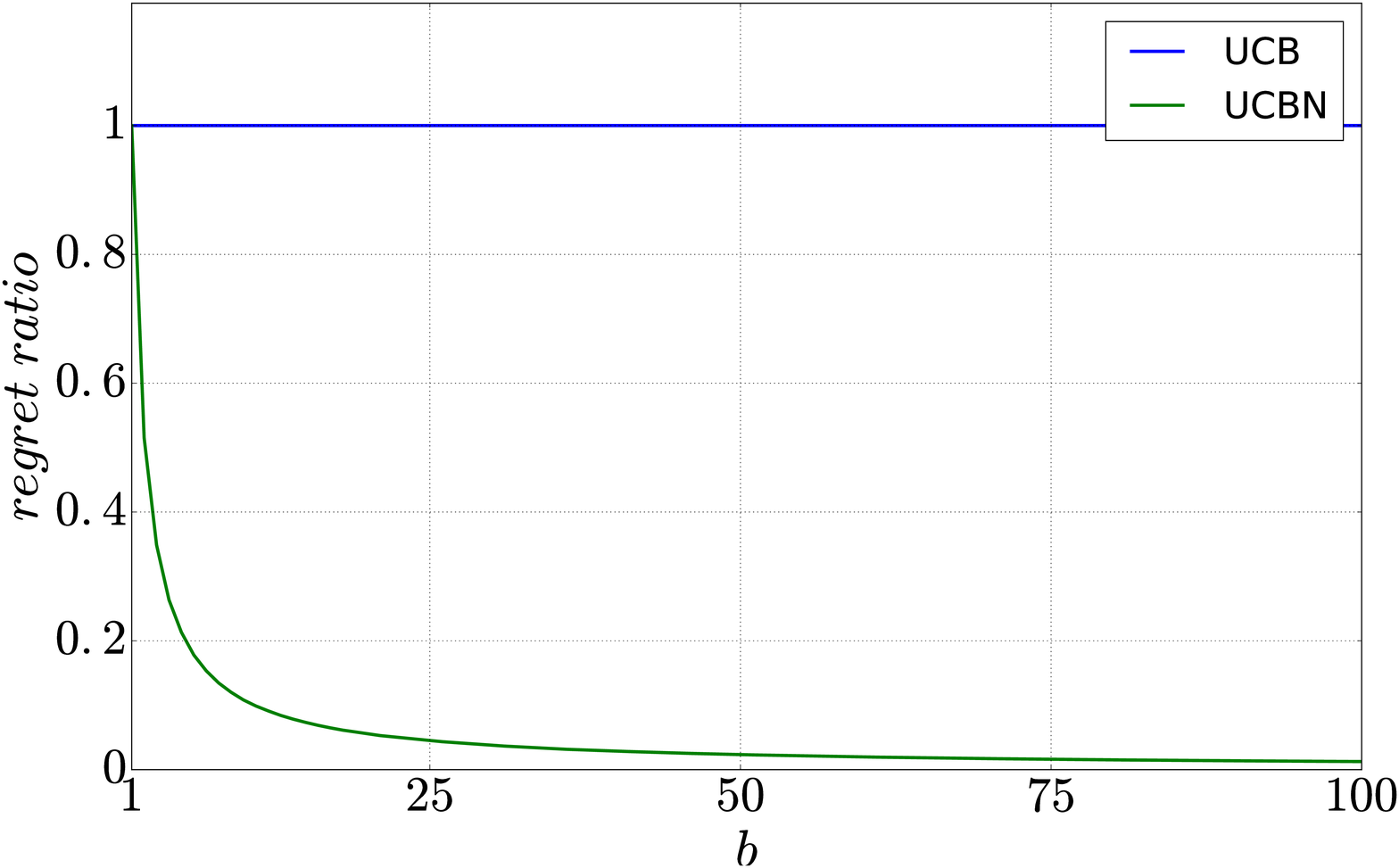}} \hfill

\captionsetup{width=\textwidth}
\caption{{Performance of our algorithm ($\UCBN$) for the stochastic setting against benchmarks. Our algorithm significantly outperforms $\UCB$, indicating that the presence of neighbors indeed improves learning. It also significantly outperforms $\GOBLIN$ that could be applied in our setting (Figure~\ref{fig:UCBcompare_GOB}). Figures (a)-(d) depict the regret. Figures (e)-(f) depict the regret ratio, i.e., the ratio between an algorithm's regret with $b$ neighbors over its regret with 0 neighbors (where $b$ depends on the network structure addressed in the corresponding subfigure).}}
	\label{fig:UCBmain}
\end{figure}

The setup for the empirical results in this section parallel that of Section \ref{sec:empirical_adv}. 
Recall that we make no assumption in our algorithm about our neighbors or how they play. We simply observe their actions and rewards. We let $\alpha = 2.5$ in the $\UCBN$ algorithm; the performance could be improved by optimizing $\alpha$. 
We first observe that more neighbors leads to less regret (Figure~\ref{fig:UCBratio_complete_b}). 

We then consider the regret of $\UCBN$ on various network topologies on 10 vertices: the complete network, a random {5}-regular network, and a star network {(Figure~\ref{fig:UCBtopology})}. Similar to the previous experiments, for networks in which all vertices have the same number of neighbors (all but the star network), all agents attain the same regret and hence we report the average regret. However, this is not the case if the number of neighbors differ; the star is the extreme example and we depict the minimum (for the center node), maximum (for one of the leaves) and average regret. As expected, the more neighbors one has, the better the regret is, with the complete network and center of star outperforming all. We also observe that there is advantage to having neighbors that are not well-connected; despite a node in the complete network having the same degree as the center node of the star, the former has slightly more regret. The reason is a neighbor with lower degree attains less information from neighbors and explore the suboptimal arms more which itself is in the favor of its neighbors (here the center of star).

We then consider the \emph{regret ratio}, i.e., the ratio between the regret of bandit algorithm $\mathcal A$ when the agent has $b$ neighbors divided by the regret of $\mathcal A$ when the agent has 0 neighbors. 
This allows us to better visualize the improvement in regret that each algorithm obtains as a function of the number of neighbors. We vary 
the number of arms $K$ (Figure \ref{fig:UCBratio_complete_K} and the number of neighbors $b$ (Figure \ref{fig:UCBratio_complete_b}). We observe that, in all cases, our algorithm $\UCBN$ attains the theoretical regret ratio, i.e, in the complete graph when all agents use $\UCBN$ the regret ration $\A_b \to \sfrac 1 b$ as $T \to \infty$.

Finally, we compare our algorithm to the one proposed in \cite{CGZ2013} ($\GOB$); see Figure \ref{fig:UCBcompare_GOB}. Although this algorithm is a centralized and developed for a different setting (namely, for {linear contextual bandits}), it can be adapted to our setting by assuming that individuals are cooperative instead of selfish. Despite the centralized nature of $\GOB$, our algorithm outperforms its regret.

\section{Conclusion \& Future Work}
In this paper, we consider a model for social learning that puts the problem in the the bandit framework
This model allows the problem to be analyzed both in the stochastic and adversarial bandit setting, and we provide algorithms for both cases. 
The regret of our algorithms interpolates between the regret of the traditional bandit setting (e.g., when an individual has no neighbors) and the regret of the full information setting (e.g., when the number of neighbors goes to infinity), and are optimal up to log factors.
We show, both theoretically and empirically, that we outperform state-of-the-art bandit algorithms that one could also apply to this setting, and illustrate how our approach could also lead to centralized algorithms of interest. 

With respect to improvements to the social learning model,  
relaxing assumption (3) would be ideal. As we have shown (see {Proposition~\ref{thm:arm_lb}}), removing it entirely results in strictly weaker regret bounds. Would an alternate relaxed assumption suffice? 
Lastly, it remains to formally study the effect of arbitrary network topologies on the regret, both for the individual (based on their position in the network) and on average.

\bibliography{bandits}
\bibliographystyle{alpha}

\appendix


\section{Adversarial Bandits}
\label{sec:adversarial}

\subsection{Regret bound for Equation \ref{regretbound_simple}}
\label{sec:oldub}

We first state and prove a slightly simpler regret upper bound that is more intuitive, and then show how to enhance the proof to give Theorem~\ref{thm:mabn_ub} in Appendix~\ref{sec:gammaub}. 
In this regret bound the algorithm is as described in the main body of the paper, and we use a fixed (as opposed to adaptive) parameters $\delta$ and $\eta$.

\begin{thm}
	\label{thm:mabn2_ub}
	Given an agent with $b$ neighbors who are playing arbitrarily, the regret of an agent using the EXPN algorithm is 
	\[{R}\in O(\sqrt{\beta T \ln K})\]
	where
	\begin{equation*}
	\Theta = \Pi_{i=1}^b \left(1-\frac{\eps_i}{K}\right)  \mbox{ and } \beta = \frac{1}{1-\left(1-\frac{1}{K}\right)\Theta} + 1.
	\nonumber
	\end{equation*}
	for an optimal choice of $\eta$ and $\delta$ that depends on $\Theta$.
\end{thm}

\noindent More precisely, we show that
\begin{equation*}\label{regretbound} 
{R}\leq \left\{
\begin{array}{ll}
2\sqrt{\beta T\ln K} & \mbox{if } \Theta \leq 1-\sqrt{\frac{\ln K}{\beta T}} , 
\mbox{ for } \eta=0, \delta = \sqrt{\frac{\ln K}{\beta T}}, \\

\sqrt{\beta T \ln K} + \sqrt{(\beta + \frac{K}{\Theta}) T \ln K} \hspace{.2in}
&  \mbox{if } 1-\sqrt{\frac{\ln K}{\beta T}}\leq \Theta \leq 1-\sqrt{\frac{\ln K}{(\beta+\frac{K}{\Theta}) T}},
\mbox{ for } \eta=0 , \delta = 1-\Theta, \\

2\sqrt{\left(\beta+\frac{K}{\Theta}\right)T\ln K} 
&   \mbox{if } 1-\sqrt{\frac{\ln K}{(\beta+ \frac K \Theta)T}} \leq \Theta \\
& \mbox{ for } \eta=\frac{K}{\Theta}\left(\sqrt{\frac{\ln K}{(\beta+\frac K \Theta)T}}+(\Theta-1)\right), \delta=\sqrt{\frac{\ln K}{(\beta+ \frac K \Theta)T}}.
\end{array} \right.
\end{equation*}

We can now directly reinterpret this regret ($R_\MABN$) as a function of the bandit regret ($R_{\EXPThree}$) and full information information regret ($R_{\MUM}$) as in Equation~\ref{regretbound_simple}.  
\noindent In particular, note that when $\Theta = 1$, none of our neighbors maintain a probability distribution that is bounded away from 0 for all arms. Hence, our neighbors are effectively not exploring. In this case, $\beta = K+1$ and  $R_{\MABN} \in O(\sqrt{TK\ln K})$, the same as in the classical bandit setting. On the other hand, when $\Theta \leq \sfrac 1 2$, then $\beta \leq 3$ and hence $R_{\MABN} \in O(\sqrt{T \ln K})$, the same as in the full-information setting. Hence, this algorithm smoothly interpolates between the bandit regret to full information regret as a function of the neighbor's exploration.

\begin{proof}
 First, note that $\mathbb{E}[\hat{g}_j(t)] =  \left(1-p^\prime_j(t)\right) \cdot0+ p^\prime_j(t) \frac{g_j(t)}{p^\prime_j(t)} = g_j(t)$; 
hence we are indeed using an unbiased estimator for the rewards. 
From the definition of $w_j(t)$ (see Section~\ref{sec:main}), we see that 
\begin{equation}\label{eql}
\ln \frac{W_{T+1}}{W_0} \geq \ln \frac{w_j(T+1)}{W_0} 
\geq \delta \sum_{t=1}^{T}\hat{g}_j(t)-\ln K.
\end{equation}
Moreover, using the definition of $p_j(t)$  (see Section \ref{sec:main}),
\begin{eqnarray*} 
\frac{W_{t+1}}{W_t} 
= \sum_{j=1}^{K}\frac{w_j(t)}{W_t}e^{\delta \hat{g}_j(t)} 
= \sum_{j=1}^{K}\left(\frac{p_j(t)-\frac{\eta}{K}}{1-\eta}\right)  e^{\delta   \hat{g}_j(t)}. 
\end{eqnarray*}
Since the algorithm selects a $\delta$ such that $\delta \hat{g}_j(t) < 1$, we can use the inequality  $e^z\leq 1+z+z^2$,  which holds for all $z \leq 1$. Therefore, 
\begin{eqnarray}
 \frac{W_{t+1}}{W_t}
    &\leq& \sum_{j=1}^{K} \left(\frac{p_j(t)-\frac{\eta}{K}}{1-\eta} \right) (1+\delta \hat{g}_j(t)+(\delta   \hat{g}_j(t))^2) \nonumber \\
    &\leq& 1 + \frac{\delta}{1-\eta}\sum_{j=1}^{K}p_j(t)\hat{g_j} + \frac{\delta^2}{1-\eta}\sum_{j=1}^{K}p_j(t)\hat{g_j}^2 .\nonumber \\\label{ineq1}
\end{eqnarray}
Taking logarithms of Equation~\eqref{ineq1} and using the inequality $\ln (1+x) < x$ which holds for all $x > 0$, we get
\begin{equation} \label{ineq}
\ln \frac{W_{t+1}}{W_t} \leq  \frac{\delta}{1-\eta}\sum_{j=1}^{K}p_j(t)\hat{g_j} + \frac{\delta^2}{1-\eta}\sum_{j=1}^{K} p_j(t)\hat{g_j}^2 . 
\end{equation}

Since 
\begin{eqnarray}
\label{eqr}
\ln \frac{W_{T+1}}{W_0} 
&=& \sum_{t=0}^T \ln \frac{W_{t+1}}{W_t}, \nonumber 
\end{eqnarray}
combining Equations \eqref{eql} and \eqref{ineq} with above and noting that $\hat g_i(0) = 0$ we have that
\begin{eqnarray} 
\delta \sum_{t=1}^{T} \hat{g}_j(t)-\ln K
\leq \frac{\delta}{1-\eta} \sum_{t=1}^{T}\sum_{j=1}^{K}p_j(t)\hat{g_j}(t) + \frac{\delta^2}{1-\eta}\sum_{t=1}^{T}\sum_{j=1}^{K}p_j(t)(\hat{g_j}(t))^2. 
		\label{ineq2}
\end{eqnarray}
Now, we note that given $p_j(t)$ at time $t$ we have 
	\begin{subequations} \label{Expectation}
		\begin{align}
		\E\left[\hat{g_j}(t)\right]&=g_j(t), \\ 
		\E\left[\sum_{t=1}^{T}\sum_{i=1}^{K} p_j(t) \hat{g}_j(t)\right] 
		&= \sum_{t=1}^{T}\sum_{j=1}^{K} p_j(t)g_j(t)=\E\left[\sum_{t=1}^{T} g_{a(t)}(t)\right], \mbox{ and }\\ 
		\E\left[\sum_{t=1}^{T}\sum_{j=1}^{K}p_j(t)\hat{g}^2_j(t)\right] 
		&= \sum_{t=1}^{T}\sum_{j=1}^{K}\frac{p_j(t)}{p^\prime_j(t)}  g^2_j(t).  \label{ex3}
		\end{align}
	\end{subequations}
Where the expectation is over randomness of the algorithm.
We will upper bound Equation \eqref{ex3}. Recall that $g_j(t) \leq 1$ for all $j,t$. Hence, for all $t$, 
\begin{subequations}
	\begin{align*}
	\sum_{j=1}^{K}\frac{p_j(t)}{p^\prime_j(t)}  g^2_j(t)
		\leq \max\left\{\sum_{j=1}^{K}\frac{p_j(t)}{p^\prime_j(t)}\right\}. 
	\end{align*}
\end{subequations}
Where the maximization is over the space of valid probabilities for actions.
This is upper bounded by Lemma~\ref{lem:UB}

\[ \beta \defeq \frac{1}{1-(1-\frac{1}{K})\Theta} + 1.\]
Since Equation \eqref{ineq2} holds for all $j$, by \eqref{Expectation}, and Lemma~\ref{lem:UB}, we get 
\begin{eqnarray*} 
 \delta \max_j \left[ \sum_{t=1}^T g_j(t) \right]   -  \ln K  \leq \frac{\delta}{1-\eta} \E\left[\sum_{t=1}^{T} g_{a(t)}(t)\right]
 	+ \frac{\delta^2}{1-\eta}  \beta T.
\end{eqnarray*}
Since $g_i(t) \leq 1$, we rearrange to get the following upper bound on the regret of our algorithm 
\begin{equation}
R \leq \frac{\ln K}{\delta} + \eta   T +  \delta \beta T. \nonumber 
\end{equation}
What remains is then an optimization problem in $\delta$ and $\eta$ 
which is subject to the following two constraints: 
\begin{equation}\label{eq:conditions}
\eta \in [0,1] \quad  \mbox{ and } \quad  \frac{\delta}{1-(1-\frac{\eta}{K})\Theta} \in [0, 1] .  
\end{equation}
In this optimization problem the only assumption made on the algorithm of agents is that they are select arms randomly and that the probability of selecting an arm has a minimum value $\eps$.  
\begin{equation}
	\begin{aligned}
		&\min_{\delta,\eta} f(\delta,\eta)= \frac{\ln K}{\delta} + \eta \cdot T +
		\delta \beta T\\
		&\text{Subject to}  \\
		&\hspace{10mm}g_1(\delta,\eta) \leq 0 ,  \\
		&\hspace{10mm}g_2(\delta,\eta) \leq 0 , \\
	\end{aligned}	
\end{equation}
where
\begin{equation}
	\begin{aligned}
		&g_1(\delta,\eta)=-\eta,\\
		&g_2(\delta,\eta)=\frac{\delta}{1-\varTheta(1-\frac{\eta}{N})}-1,
	\end{aligned}
\end{equation}
and
\begin{equation} \label{param}
	\begin{aligned}
		&\varTheta=\Pi_{i=2}^b(1-\frac{\eta_i}{K}),\\
		&\beta=\frac{1}{1-(1-\frac{1}{K})\varTheta}+1,
	\end{aligned}
\end{equation}

If $\delta^\star$ and $\eta^\star$ is a local minimum that satisfies Karush-Kuhn-Tucker (KKT) conditions(see below).\\
Stationary:
\begin{equation}
	-\nabla f(\delta^\star,\eta^\star)= \mu_1\cdot \nabla g_1(\delta^\star,\eta^\star)+\mu_2\cdot \nabla g_2(\delta^\star,\eta^\star),
\end{equation}
Primal feasibility:
\begin{equation}
	\begin{aligned}
		&g_1(\delta^\star,\eta^\star) \leq 0 \\
		&g_2(\delta^\star,\eta^\star) \leq 0 , \\
	\end{aligned}
\end{equation}
Dual feasibility:
\begin{equation}
	\begin{aligned}
		\mu_1\geq 0\\
		\mu_2\geq 0\\
	\end{aligned}
\end{equation}
Complementary slackness:
\begin{equation}
	\begin{aligned}
		\mu_1\cdot g_1(\delta^\star,\eta^\star) = 0\\
		\mu_2\cdot g_2(\delta^\star,\eta^\star) = 0\\
	\end{aligned}
\end{equation}
In each step we will assume that some of these constraints are active (i.e.,$g_i(\delta,\eta) = 0$), and find the points that satisfy the KKT conditions,
\begin{equation}
	-\nabla f(\delta,\eta)=
	\begin{pmatrix}
		\frac{\ln K}{\delta^2}-T\beta\\
		-T
	\end{pmatrix}
\end{equation}
First, let us assume that only first constraint is active,
\[g_1(\delta,\eta)=0.\]
Which yields that
\[\eta^\star=0.\] 
The only stationary point in this case is
\begin{equation}
	\delta^\star=\sqrt{\frac{\ln K}{\beta T}}.
\end{equation} 
And the $\varTheta$ that primal and dual feasibility holds for is 
\begin{equation}\label{b1}
	\varTheta\leq1-\sqrt{\frac{\ln K}{\beta T}}.
\end{equation}  
Since $\delta^\star$ and $\eta^\star$ satisfy KKT conditions they are valid answers for this interval.
The regret is
\begin{equation}
	R\leq \sqrt{\beta \ln K T}.
\end{equation} 
Second, let us assume that only the second constraint is active.
In this case the stationary point is 
\begin{equation}
	\begin{aligned}
		&\delta^\star=1-\varTheta(1-\frac{\eta}{K})\\
		&\eta^\star=\frac{K}{\varTheta}\sqrt{\frac{\ln K}{(\beta+K/\varTheta)T}}+\frac{K(\varTheta-1)}{\varTheta}.
	\end{aligned}
\end{equation}
The interval that this answer is valid for is as follows,
\begin{equation}\label{b2}
	\begin{aligned}
		&\varTheta \geq 1-\sqrt{\frac{\ln K}{(\beta+\frac{K}{\varTheta}) T}}.
	\end{aligned}
\end{equation}
The regret is
\begin{equation}
	R\leq 2\sqrt{(\beta+\frac{K}{\varTheta})T\ln K}+\frac{K(\varTheta-1)T}{\varTheta}.
\end{equation}

Finally, let us assume that all of the constraints are active.
In this case the stationary point is
\begin{equation}
	\begin{aligned}
		&\delta^\star=1-\varTheta\\
		&\eta^\star=0 \mbox{ .}
	\end{aligned}
\end{equation}
And the interval that this answer is valid is as follows,
\begin{equation}\label{b3}
	\begin{aligned}
		&\varTheta \geq 1-\sqrt{\frac{\ln K}{\beta T}},\\
		&\varTheta \leq 1-\sqrt{\frac{\ln K}{(\beta+\frac{K}{\varTheta}) T}}\mbox{ .}
	\end{aligned}
\end{equation}
The regret is
\begin{equation}
	R\leq \frac{\ln K}{1-\varTheta} + (1-\varTheta)\beta T \mbox{ .}
\end{equation}
After solving the inequalities \eqref{b1}, \eqref{b2} and \eqref{b3} we get the desired bound for the regret.
\end{proof}

\begin{lem}\label{lem:UB} 
$\sum_{j=1}^{K}\frac{p_j(t)}{p^\prime_j(t)} \leq \frac{1}{1-(1-\frac{1}{K})\Theta} + 1 \defeq \beta$ for all $t$.
\end{lem}
\begin{proof}[Proof of Lemma~\ref{lem:UB}]
By definition (see Equations \ref{eq:MABNestimator} and \ref{pPrime}), 
\begin{equation}
S \defeq \sum_{j=1}^{K}\frac{p_j(t)}{p^\prime_j(t)} 
  = \sum_{j=1}^{K}\frac{p_j(t)}{1- (1-p_j(t)) \cdots (1-q^b_j(t))}.
\nonumber
\end{equation}
Since $q^i_j \geq \sfrac {\eps_i}{K}$, clearly
\begin{equation}
S \leq \sum_{j=1}^{K}\frac{p_j(t)}{1- (1-p_j(t))\Theta}.
\nonumber
\end{equation}

Notice that the right-hand side is strictly concave with respect to each $p_i$. Thus, the maximum is achieved either on the boundary of the feasible region of $p$, or at a single point in the interior. 
Recall that $p_i \in [\sfrac{\eta}{K}, 1-(K-1)\frac{\eta}{K}]$, and that $p$ is a probability distribution. Hence, boundary points are of the form $\sfrac{\eta}{K}$ for all except one entry, which is $1-(K-1)\frac{\eta}{K}$.
If $p$ is of this form, then

\begin{eqnarray}
S \leq \frac{(K-1)}{K} \cdot \frac{\eta}{1-(1-\frac{\eta}{K})\Theta} +\frac{1-\frac{(K-1)\eta}{K}}{1-\frac{(K-1)\eta \Theta}{K}} 
\leq \frac{\eta}{1-(1-\frac{\eta}{K})\Theta} + 1.
\label{eq:UB1}
\end{eqnarray}

If, instead, the maximum is achieved on the interior, it must be symmetric, i.e., the probability of playing all actions is $\sfrac 1 K$; otherwise, since the equation is symmetric, there would be more than one maximal distribution which contradicts strict concavity. If $p$ is of this form, then
\begin{equation}
S \leq \frac{1}{1-(1-\frac{1}{K})\Theta}. \label{eq:UB2}
\end{equation}
Since $\eta \in [0,1]$, we can upper bound both the right-hand sides of Equations \ref{eq:UB1}  and \ref{eq:UB2} to attain the desired bound: 
\begin{equation} 
\sum_{i=1}^{K}\frac{p_i(t)}{p^\prime_i(t)} \leq \frac{1}{1-(1-\frac{1}{K})\Theta} + 1.
\nonumber
\end{equation}
\end{proof}

\subsection{Proof of Main Theorem~\ref{thm:mabn_ub}}
\label{sec:gammaub}

For this result, we use the same unbiased estimator as presented in the main body of the paper, but take a different approach in the analysis. Instead of having an explicit exploration parameter $\eta>0$, we instead will allow the update parameter $\delta_t$ to vary with time $t$ (in fact, in this case $\eta = 0$). This sort of analysis is common (see, e.g., \cite{BanditBook} and \cite{KNVM2014}), in particular in settings where $T$ is unknown. In our case the adaptive $\delta_t$ functions to incorporate information from neighbors as we see it.

\begin{proof}[Proof of Main Theorem~\ref{thm:mabn_ub}]
	The first part of proof (from Equation~\eqref{eq:eq1} to Equation~\eqref{eq:main}) parallels the format of the proof of Theorem 3.1 in \cite{BanditBook} for the usual bandit setting.
	We first write the $\mathbb{E}[\hat{l}_j(t)]$ as Equation~\eqref{eq:eq1}, then we upper bound the first term and lower bound the second term, which leads to Equation~\eqref{eq:main}. 
	\begin{equation}\label{eq:eq1}
	\mathbb{E}[\hat{l}_j(t)]=\frac{1}{\delta}\left(\ln \mathbb{E}\left[\exp\left(-\delta\left(\hat{l}_j(t)-\mathbb{E}[\hat{l}_j(t)]\right)\right)\right]
	-\ln \mathbb{E}\left[\exp\left(-\delta\hat{l}_j(t)\right)\right]\right).
	\end{equation}
	where the expectation is over the randomness of the estimator and choice of the arm.
	
	We now diverge from the usual proof template: In the next step we find an upper bound for the first term in right-hand side of above equation.
	\begin{equation}
	\begin{aligned}
	\frac{1}{\delta}\ln \mathbb{E}\left[\exp\left(-\delta(\hat{l}_j(t)-\mathbb{E}[\hat{l}_j(t)])\right)\right] &=
	\frac{1}{\delta}\ln \mathbb{E}\left[\exp(-\delta\hat{l}_j(t)\right]+\mathbb{E}[\hat{l}_j(t)]) \\&\leq
	\frac{1}{\delta}\mathbb{E}[\exp(-\delta\hat{l}_j(t))-1+\delta\hat{l}_j(t)] \\&\leq
	\frac{\delta}{2}\mathbb{E}[\hat{l}^2_j(t)],
	\end{aligned}
	\end{equation}
	where in the second inequality we use $\ln x\leq x-1$ and in the last inequality we use $\exp(-x)-1+x\leq x^2/2$ for $x\geq 0$.
	Defining $\hat{L}_j(t)=\sum_{a=1}^{t}\hat{l}_j(t)$ we have
	\begin{equation}\label{eq:eq2}
	\begin{aligned}
	-\frac{1}{\delta} \ln \mathbb{E}_{a_t\sim p^\prime_t}\mathbb{E}_{j\sim p^\prime_t} &\exp(-\delta \hat{l}_j(t))  \leq
	-\frac{1}{\delta} \mathbb{E}_{a_t\sim p^\prime_t} \ln \mathbb{E}_{j\sim p^\prime_t} \exp(-\delta \hat{l}_j(t)) \\ & =
	-\frac{1}{\delta} \mathbb{E}_{a_t\sim p^\prime_t} \ln \sum_{j=1}^{K}p^\prime_j(t) \exp(-\delta \hat{l}_j(t)) \\ &=
	-\frac{1}{\delta} \mathbb{E}_{a_t\sim p^\prime_t}\left[ \ln\left( \sum_{j=1}^{K}(1-\eta) \frac{\exp(-\delta \hat{L}_j(t))}{\sum_{c=1}^{K}\exp(-\delta \hat{L}_c(t-1))}+\frac{\eta}{K}\exp(-\delta\hat{l}_j(t))\right)\right] \\&\leq   
	-\frac{1}{\delta}\mathbb{E}_{a_t\sim p^\prime_t}\left( \frac{\eta}{K}\sum_{j=1}^K(-\delta\hat{l}_j(t)) +(1-\eta)\left[ \ln\left(\sum_{j=1}^{K}\frac{\exp(-\delta \hat{L}_j(t))}{\sum_{c=1}^{K}\exp(-\delta \hat{L}_c(t-1))}\right)\right]       \right)     ,                           
	\end{aligned}
	\end{equation} 
	where in the first and last inequality we used the Jensen's inequality.
	Now we want to upper bound the term $-\frac{1}{\delta}\ln\left(\sum_{i=j}^{K}\frac{\exp(-\delta \hat{L}_j(t))}{\sum_{c=1}^{K}\exp(-\delta \hat{L}_c(t-1))}\right)$ in the above inequality.
	\begin{equation} \label{eq:eq3}
	\begin{aligned}
	-\frac{1}{\delta}\ln\left(\sum_{j=1}^{K}\frac{\exp(-\delta \hat{L}_j(t))}{\sum_{c=1}^{K}\exp(-\delta \hat{L}_c(t-1))}\right) = \psi(t-1)-\psi(t),
	\end{aligned}
	\end{equation}
	where $\psi(t)=\frac{1}{\delta}\ln\left(\frac{1}{K}\sum_{c=1}^{K}\exp(-\delta \hat{L}_c(t))\right)$.
	By summing up these terms in Equation~\eqref{eq:eq1},~\eqref{eq:eq2}, and~\eqref{eq:eq3} we get
	\begin{equation}
	\sum_{t=1}^{T}\mathbb{E}[\hat{l}_j(t)] \leq \sum_{t=1}^{T}\frac{\delta}{2} \mathbb{E}[\hat{l}^2_j(t)]+\sum_{t=1}^{T}\frac{\eta}{K} \mathbb{E}_{a_t\sim p^\prime_t} \sum_{c=1}^K\hat{l}_c(t)  -\mathbb{E}_{a_t\sim p^\prime_t}\psi(T).
	\end{equation}
	By bounding $-\psi(T)$ we are able to find an upper bound on the regret.
	\begin{equation}
	\begin{aligned}
	-\psi(T) &=  \frac{\ln K}{\delta} - \frac{1}{\delta} \ln \left(\sum_{j=1}^{K}\exp(-\delta \hat{L}_j(T))\right) \\&\leq
	\frac{\ln K}{\delta} - \frac{1}{\delta} \ln \left(\exp(-\delta \hat{L}_k(T))\right)  \\&=
	\frac{\ln K}{\delta} + \hat{L}_k(T).
	\end{aligned}
	\end{equation}
	Plugging this in the above inequality yields
	\begin{equation} \label{eq:main}
	\sum_{t=1}^{T}\mathbb{E}[\hat{l}_j(t)] \leq \sum_{t=1}^{T}\frac{\delta}{2} \mathbb{E}[\hat{l}^2_j(t)]+\sum_{t=1}^{T}\frac{\eta}{K} \mathbb{E}_{a_t\sim p^\prime_t} \sum_{c=1}^K\hat{l}_c(t)+\frac{\ln K}{\delta}+\mathbb{E}_{a_t\sim p^\prime_t} [\hat{L}_k(T)].
	\end{equation}
	In fact, having more neighbors help us to get better bounds for $\mathbb{E}[\hat{l}^2_j(t)]$.
	
	Now, we note that given $p_j(t)$ at time $t$ we have 
		\begin{subequations} \label{Expectation}
		\begin{align}
		\E\left[\hat{g_i}(t)\right]&=g_i(t), \\ 
		\E\left[\sum_{t=1}^{T}\sum_{i=1}^{K} p_i(t) \hat{g}_i(t)\right] 
		&= \sum_{t=1}^{T}\sum_{i=1}^{K} p_i(t)g_i(t)=\E\left[\sum_{t=1}^{T} g_{a(t)}(t)\right], \mbox{ and }\\ 
		\E\left[\sum_{t=1}^{T}\sum_{i=1}^{K}p_i(t)\hat{g}^2_i(t)\right] 
		&= \sum_{t=1}^{T}\sum_{i=1}^{K}\frac{p_i(t)}{p^\prime_i(t)}  g^2_i(t).  \label{ex3}
		\end{align}
	\end{subequations}
	Where the expectation is over randomness of the algorithm. Note that $g_i(t)$ is less than 1, as a result we have,
	\[\sum_{t=1}^{T}\sum_{i=1}^{K}\frac{p_i(t)}{p^\prime_i(t)}  g^2_i(t) \leq \sum_{t=1}^{T}\sum_{i=1}^{K}\frac{p_i(t)}{p^\prime_i(t)} .\]
	Setting $\eta=0$ and using the Lemma~\ref{lem:helper} we have
	\begin{equation}\label{eq:regret_bound}
	R \leq \frac{\ln K}{\delta_T} +   \frac{1}{2}  \sum_{t=1}^{T}\delta_t(1+\gamma_t).
	\end{equation}
	To conclude the proof, let $\delta_t=\sqrt{\frac{\ln K}{\sum_{c=1}^{t}(1+\gamma_t)}}$ and use Lemma 3.5 of \cite{auer2002adaptive}, which gives an upper bound on the regret to get the expected regret, simply take expectation of the both sides of Equation~\eqref{eq:regret_bound}. %

\end{proof} 

 \begin{lem} \label{lemma1}
 	For $p^i \in [0,1]$ we want to show that $\Pi_{i=1}^b (1-p^i) \leq \frac{1}{1+\sum_{i=1}^{b}p^i}$.\end{lem}
 \begin{proof}
 	The following formula holds for all positive $p_i$,
 	\begin{equation}
 	\Pi_{i=1}^b (1+p^i) \geq 1+\sum_{i=1}^{b}p^i,
 	\end{equation}
 	which implies
 	\begin{equation}\label{l1}
 	\frac{1}{\Pi_{i=1}^b (1+p^i)}   \leq  \frac{1}{1+\sum_{i=1}^{b}p^i}.
 	\end{equation}
 	As we have
 	\begin{equation}
 	\begin{aligned}
 	&\Pi_{i=1}^b (1-p^i) \cdot \Pi_{i=1}^b (1+p^i)\\
 	&=\Pi_{i=1}^b (1-(p^i)^2) \leq 1,
 	\end{aligned}
 	\end{equation}
 	We can conclude  
 	\begin{equation}\label{l2}
 	\Pi_{i=1}^b (1-p^i) \leq \frac{1}{\Pi_{i=1}^b (1+p^i)}.
 	\end{equation}
 	From \eqref{l1} and \eqref{l2} we get
 	\begin{equation}
 	\begin{aligned}
 	\Pi_{i=1}^b (1-p^i) & \leq \frac{1}{\Pi_{i=1}^b (1+p^i)}\\
 	& \leq \frac{1}{1+\sum_{i=1}^{b}p^i}.
 	\end{aligned}
 	\end{equation}
 \end{proof}
 \begin{lem}\label{lem:helper}
 	$\sum_{j=1}^{K} \frac{p_j}{1-(1-p_j)(1-q^1_j)\cdots(1-q^b_j)} \leq \sum_{j=1}^{K}\frac{p_j}{p_j+q^1_j+q^2_j+\cdots+q^b_j}+1$.
 \end{lem}
 \begin{proof}
 	From Lemma~\ref{lemma1} we have
 	\begin{equation}
 	\begin{aligned}
 	(1-p_j)(1-q^1_j)\cdots(1-q^b_j) 
 	\leq \frac{1}{1+p_j+q^1_j+q^2_j+\cdots+q^b_j},
 	\end{aligned}
 	\end{equation}
 	substituting this bound in the statement of lemma yields
 	\begin{equation}
 	\begin{aligned}
 	\sum_{j=1}^{K} \frac{p_j}{1-(1-p_j)(1-q^1_j)\cdots(1-q^b_j)} &\leq \sum_{j=1}^{K} \frac{p_j}{1-\frac{1}{1+p_j+q^1_j+q^2_j+\cdots+q^b_j}}\\
 	&\leq \sum_{j=1}^{K}\frac{p_j}{p_j+q^1_j+q^2_j+\cdots+q^b_j}+1,
 	\end{aligned}
 	\end{equation}
 \end{proof}

\subsection{Proof of Regret Lower Bound} \label{sec:lowerbound}	
	\begin{proof}[Proof of Theorem~\ref{thm:gen_lb}]		
		This proof follows, with some adjustments, from the proof of Theorem 5.1 in \cite{ACFS2003}; we simply point out the key differences here. In this proof a random distribution of rewards is constructed, one of the actions is chosen uniformly at random to be the ``good'' action, which is 1 with probability $1/2+\epsilon$ and 0 otherwise for some small fixed $\epsilon\in (0,1/2]$. The remaining actions are 0 or 1 with probability $1/2$. 
		
		In Equation~(30) of \cite{ACFS2003}, where the relative entropy between two Bernoulli random variables with parameters $P$ and $Q$ is calculated, the total number of past observation is $T$. In our case, it is between $\sum_{t=1}^{T}n_t$ and $T+\sum_{t=1}^{T}n_t$ where $n_t$ is the size of the set of arms selected (arbitrarily) by all of the individual?s neighbors at time $t$. Note that, stopping here, this would lead directly to the proof of Theorem~\ref{thm:arm_lb}. Here, we dig into the proof to break up the regret in a way that allows us to write the regret with respect to our $\gamma_t$, which will result in a tighter bound.
		
		 More precisely, let $r$ denote the entire sequence of rewards, let $f$ is an arbitrary function on such histories $r$ to a fixed range $[0,T]$, and let $N_i$ be a random variable denoting the number of times an algorithm selects arm $i$. We let $\mathbb{E}_i$ be the expectation taken conditioned on $i$ being the good action, and $\mathbb{E}_{unif}$ denotes the expectation taken over a uniformly random choice of rewards for all actions (including the good action). Then, 
 Lemma~\textbf{A.1} in \cite{ACFS2003} becomes 
		\[\mathbb{E}_i[f(r)] \leq \mathbb{E}_{unif}[f(r)] + \frac{T}{2}\sqrt{-(\mathbb{E}_{unif}[N_i])\ln (1-4\epsilon^2)}, \]
		when $f(r)$ is the number of times that we choose arm $i$. 
		Note that, for our algorithm, $\mathbb{E}_{unif}[N_i] = \sum_{t=1}^{T}1-(1-p_i(t))(1-q_i^1(t))\cdots(1-q_i^b(t))$. 
		Using the lemma, if we set $\epsilon=c\sqrt{\frac{K}{\sum_{i=1}^{K}\mathbb{E}_{unif}[N_i]}}$, and continue to follow the proof  of Theorem 5.1 in \cite{ACFS2003}. By observing that  $\Omega\left(T\sqrt{\frac{K}{\sum_{i=1}^{K}\mathbb{E}_{unif}[N_i]}  }\right) = \Omega\left(T+\sum_{t=1}^{T}\gamma_t\right)$, 
		we recover the desired lower bound. 		
	\end{proof}

\subsection{Comparison to Arm-Network Algorithms}
\label{sec:comparison}

We compare our algorithm to arm-network algorithms which, though developed for a different setting, could be applied to ours. These algorithms are given for a multi-armed-bandit (MAB) with side observations encoded as an arm-network. In these papers, authors assume the arms form a arm-network $G_t$ (the arm-network can change over time). The arm-network is a directed graph, whose vertices represent the arms, and an edge from arm $i$ to arm $j$ means that by choosing arm $i$ we observe the reward of arm $j$ (see Figure~\ref{fig:arm_network}).
First, let recall our interpretation of our problem as a multi-armed bandit problem with an arm-network (see Figure~\ref{fig:arm_network}). 
The players choose arms according to their algorithm, and after selecting and revealing the rewards. Each player $i$ individually can construct a arm-network. If we show the arms selected by neighbors of player $i$ by $\mathcal{A}[N(i)]$, in arm-network there is an edge from all arms (vertices) to $\mathcal{A}[N(i)]$. In addition, if we denote the cardinality of set $\mathcal{A}[N(i)]$ by $C_i$, the independence number $\alpha$ of the underlying graph is $K+1-C_i$.

We show that our algorithm's regret is at most that of $\EXPG$, the state-of-the-art algorithm for the arm-network setting. 
From Theorem ~\ref{thm:mabn_ub}, we know that that the regret of our algorithm is $\tilde{O}\left(\sqrt{T+\sum_{t=1}^{T}\gamma_t}\right)$ and from \cite{ACDK2015} the regret of $\EXPG$ is $\tilde{O}\left(\sqrt{\sum_{t=1}^{T}\alpha_t}\right)$.

 With some abuse of notation, let $\mathcal{A}$ be the arms chosen by neighbors of $0$  let $C$ be its cardinality, and let $1\{j\in \mathcal{A}\}$ be an indicator random variable that is 1 if and only if $j$ is in $\mathcal{A}$. Lastly, let $1\{j = a^\ell_t\}$ be an indicator random variable that is 1 if and only if player $\ell$ (one of the neighbors) chooses arm $j$ at round $t$.
First, we lower bound $\alpha$ using the indicator random variable defined above, then we show that in expectation this term is greater than $\gamma_t$.
\begin{lem} \label{lem:bound}
The independence number $\alpha$ is lower bounded as follows	
\begin{equation}
\alpha_t \geq \sum_{j=1}^{K}\frac{p_j(t)}{p_j(t)+\sum_{\ell=1}^{b}1\{j = a^\ell_t\}} .
\end{equation}
\end{lem}
\begin{proof}
	First, we show the following 
	\begin{equation}
	\alpha_t \geq \sum_{j=1}^{K}\frac{p_j(t)}{p_j(t)+1\{j\in \mathcal{A}\}}.
	\end{equation}
	Decomposing the sum we have
	\begin{equation}
	\sum_{j=1}^{K}\frac{p_j(t)}{p_j(t)+1\{j\in \mathcal{A}\}} = \sum_{j\in \mathcal{A}}\frac{p_j(t)}{p_j(t)+1} + \sum_{j\notin \mathcal{A}}\frac{p_j(t)}{p_j(t)},
	\end{equation}
	plugging the cardinality of $\mathcal{A}$ we get
	\begin{equation}
	 = \sum_{j\in \mathcal{A}}\frac{p_j(t)}{p_j(t)+1} + K-C ,
	\end{equation}
	we know $\alpha=1+K-C$ (this comes from the fact that arms in $\mathcal{A}$ are connected to all arms, see Figure~\ref{fig:arm_network}), knowing that $\sum_{j\in \mathcal{A}}\frac{p_j(t)}{p_j(t)+1}$ is less than 1 completes the first part of lemma.
	
	Second, we have $1\{j\in \mathcal{A}\} \leq \sum_{\ell=1}^{b}1\{j = a^\ell_t\}$, which yields the lemma.
\end{proof}

In the Lemma~\ref{lem:bound}, the term $1\{j = a^\ell_t\}$ can be seen as an unbiased estimator for $q_j^\ell(t)$ (we denote it by $\hat{q}_j^\ell(t)$). As a final step, we show the following lemma.
\begin{lem} \label{lem:alpha_gamma}
	For a multinomial distribution q and their unbiased estimator $\hat{q}^\ell_j(t) = 1\{j = a^\ell_t\}$, we have
		\begin{equation}
			\sqrt{\sum_{t=1}^{T}(1+\gamma_t)} \leq \mathbb{E}_{\hat{q}\sim q}\left[\sqrt{2\sum_{t=1}^{T}\alpha_t}\right]
		\end{equation}
\end{lem}
\begin{proof}
	Because $\alpha_t \geq 1$, we know $\sum_{t=1}^{T}\alpha_t \geq T$. As a result by showing $\sqrt{\sum_{t=1}^{T}(\gamma_t)} \leq \mathbb{E}_{\hat{q}\sim q}\left[\sqrt{\sum_{t=1}^{T}\alpha_t}\right]$, we can conclude $\sqrt{T+\sum_{t=1}^{T}(\gamma_t)} \leq \mathbb{E}_{\hat{q}\sim q}\left[\sqrt{2\sum_{t=1}^{T}\alpha_t}\right]$. 
	
	From Lemma~\ref{lem:bound}, we have
	\begin{equation}
	\mathbb{E}_{\hat{q}\sim q}\left[\sqrt{\sum_{t=1}^{T}\alpha_t}\right] \geq 
	\mathbb{E}_{\hat{q}\sim q}\left[ \sqrt{\sum_{t=1}^{T}\sum_{j=1}^{K}\frac{p_j(t)}{p_j(t)+\sum_{\ell=1}^{b}\hat{q}^\ell_j(t)}} \right].
	\end{equation}
	In the next step, we want to show 
	\[ \mathbb{E}_{\hat{q}\sim q}\left[ \sqrt{\sum_{t=1}^{T}\sum_{j=1}^{K}\frac{p_j(t)}{p_j(t)+\sum_{\ell=1}^{b}\hat{q}^\ell_j(t)}} \right] \geq \sqrt{\sum_{t=1}^{T}\gamma_t}. \]
	 Let $\phi(\hat{q})=\phi(\hat{q}^1(1),\hat{q}^2(1),\cdots,\hat{q}^b(1),\hat{q}^1(2),\cdots,\hat{q}^b(T)) = \sqrt{\sum_{t=1}^{T}\sum_{j=1}^{K}\frac{p_j(t)}{p_j(t)+\sum_{\ell=1}^{b}\hat{q}^\ell_j(t)}}$. This function is convex (it is convex along every arbitrary line with positive entires, so it is convex), we can use Jensen's inequality to swap the order of expectation and $\phi$ to get the following.
	 
	 \begin{align}
	 \mathbb{E}_{\hat{q}\sim q} \left[\phi(\hat{q})\right] \geq  \phi\left(\mathbb{E}_{\hat{q}\sim q} \left[\hat{q}\right]\right) = \phi\left( q \right) = \sqrt{\sum_{t=1}^{T}\gamma_t}.
	 \end{align}
	 The first inequality is Jensen's inequality and the last equality comes from definition of $\gamma_t$.
\end{proof}

\section{Stochastic Bandits}
\label{sec:stochastic}

In this section, we first formally state and prove the results for the stochatic setting. 
Let us first recall our result:
\begin{thm} [Theorem~\ref{thm:UCBN_ub}]
Consider an agent with neighbors who play arbitrarily.  Let $n_i^\prime(t)$ be the number times arm $i$ has been selected by one of her neighbors by time $t$. %
	Then, the regret  of $\UCBN$ for any  $\alpha >2$ is
	\begin{equation}\label{bound2}
		\begin{aligned}
			R \leq  \sum_{i,\Delta_i>0}\left(\max \left \{\max_{t=1,..,T} \left\{\frac{2\alpha\ln t}{\Delta_i}-n^\prime_i(t)\Delta_i\right\} , 0  \right\} + \frac{\alpha}{\alpha-2}\right), 
		\end{aligned}
	\end{equation}
	where $\Delta_i$ is the difference between $\mu_{i^\star}$ and  $\mu_i$.
\end{thm}

\begin{cor}	\label{allUCBN}
	On a complete graph with $b$ nodes, if all agents use $\UCBN$ then under the same conditions as in Theorem~\ref{thm:UCBN_ub}, the regret of an agent is 
	\begin{equation}\label{bound}
		R \leq \sum_{i,\Delta_i>0} \left(\frac{2 \alpha \ln T}{b \Delta_i}+ \frac{\alpha}{\alpha-2}\right)
		\in O\left(\frac{K\ln T}{b}\right).
	\end{equation}
\end{cor}
The following lower bound yields same behavior for getting free observation from neighbors.
\begin{thm}\label{thm:UCBN_LB}
	Consider a strategy that satisfies $\E\mathbb{[}n_i(T)\mathbb{]} = o(T^a)$, any arm $i$ with $\Delta_i > 0$, and any $a > 0$. Let $c_t(i)$ be the number of times arm $i$ selected (arbitrarily) by all of the agent's neighbors up to time $t$, then, for any set of Bernoulli reward distributions the following inequality holds
	\begin{equation}
		\lim_{T\longrightarrow +\infty} \inf
		\frac{R}{\ln T}\geq \sum_{i,\Delta_i>0} \frac{1}{2\Delta_i} - \lim_{T\longrightarrow +\infty} \inf
		\frac{\sum_{i,\Delta_i>0} c_T(i)\Delta_i}{\ln T}
		.
	\end{equation}	
\end{thm}

Our proofs parallel, with additional bookkeeping, the proofs for the original UCB results (see, e.g., \cite{BanditBook} for a template).
Note that the results for stochastic bandits hold when the reward distributions satisfy the following standard conditions.
\begin{defn}[Conditions on $\F_i$]\label{conditions}
Every reward distribution $\F_i$ satisfies Hoefding's lemma, i.e., there exists a convex function $\psi$ on the reals such that, for all $\lambda \geq 0$, we have $\ln\left[ \E \left[ e^{\lambda \left| X - \E[X] \right|} \right] \right] \leq \psi(\lambda)$ where $X \sim \F_i$. 
\end{defn}
For example, when $X \in [0,1]$, one can take $\psi(\lambda) = \frac{\lambda^2}{8}$; indeed the results in the main body of the paper take this $\psi$. The results can be easily generalized for other $\psi$ in the usual manner. 

We first prove a lemma that will be of assistance in the proof of Theorem~\ref{thm:UCBN_ub}. Recall that $a(t)$ is the arm the agent selects at time $t$. 

\begin{lem} \label{lema3}
	If $a(t)=i$, at least one of the three following inequalities is true:
	\begin{subequations}
		\begin{align}
		& \hat{\mu}_{i^\star,n_{i^\star}(t-1)}+\sqrt{\frac{\alpha \ln t}{2n_{i^\star}(t-1)}} \leq \mu^\star, \label{3cond1}\\
		& \hat{\mu}_{i,n_i(t-1)} > \mu_i + \sqrt{\frac{\alpha \ln t}{2n_{i}(t-1)}},\label{3cond2}\\
		& n_i(t-1) < \frac{2\alpha \ln t}{\triangle_i^2}. \label{3cond3}
		\end{align}
	\end{subequations}
\end{lem}

\begin{proof}
	We prove the contrapositive. Assume that $a(t) = i$ and that none of the inequalities \eqref{3cond1}, \eqref{3cond2} or \eqref{3cond3} are true.
	\begin{subequations} 
		\begin{align}
		& \hat{\mu}_{i^\star,n_{i^\star}(t-1)}+\sqrt{\frac{\alpha \ln t}{2n_{i^\star}(t-1)}} > \mu^\star, \label{3false1}\\
		& \hat{\mu}_{i,n_i(t-1)} < \mu_i +\sqrt{\frac{\alpha \ln t}{2n_{i}(t-1)}},\label{3false2}\\
		& n_i(t-1) > \frac{2\alpha \ln t}{\triangle_i^2}. \label{3false3}
		\end{align}
	\end{subequations}
	By plugging $\mu^\star=\mu_i+\triangle_i$ in Equation \eqref{3false1} we obtain
	\begin{equation}\label{u1}
	\hat{\mu}_{i^\star,n_{i^\star}(t-1)}+\sqrt{\frac{\alpha \ln t}{2n_{i^\star}(t-1)}} > \mu_i+\triangle_i.
	\end{equation}
	From Equation \eqref{3false3} we have
	\begin{equation}\label{proof2}
	\mu_i+\triangle_i > \mu_i + \sqrt{\frac{2\alpha \ln t}{n_{i}(t-1)}}
	\end{equation} 
	and plugging \eqref{proof2} in \eqref{u1} yields
	\begin{equation}
	\hat{\mu}_{i^\star,n_{i^\star}(t-1)}+\sqrt{\frac{\alpha \ln t}{2n_{i^\star}(t-1)}}> \hat{\mu}_{i,n_i(t-1)}+\sqrt{\frac{\alpha \ln t}{2n_{i}(t-1)}}.
	\end{equation}
	By our criteria for selecting arms given by Equation \eqref{rule}, this implies $a(t)\neq i$.
\end{proof}

\begin{proof}[Proof of Theorem \ref{thm:UCBN_ub} ]
Let 
\begin{equation}\label{T}
n_i(t)=n^1_i(t)+n^\prime_i(t),
\end{equation}
where $n^1_i(t)$ is the number of times the agent selects the arm $i$ and $n^\prime_i(t)$ is the number of times her neighbors select arm $i$. Using Lemma \ref{lema3}, we will first find an upper bound for $n^1_i(t)$ for a suboptimal arm $i$. 

Lemma \ref{lema3} states that at least one of the three inequalities \eqref{3cond1}, \eqref{3cond2} and \eqref{3cond3} must be true. 
If Equation \eqref{3cond3} holds, then from \eqref{T} we obtain
\begin{equation}\label{upT}
	n^1_i(t) \leq \frac{2\alpha}{ \triangle_i^2}\ln t-n^\prime_i(t).
\end{equation} 
Let $U$ be the maximum of right hand side of \eqref{upT} for $t=1,..,T$:
\begin{equation}\label{U}
	U=\max\left\{\max_{t=1,..,T} \left\{\frac{2\alpha}{ \triangle_i^2}\ln t-n^\prime_i(t)\right\}\triangle_i , 0  \right\},
\end{equation}
as a result if the Equation \eqref{3cond3} holds for some instance $k$, then $n^1_i(k)$ is bounded by $U$, i.e.,
\begin{equation} \label{first}
	n^1_i(k) \leq U.
\end{equation} 
For bounding the regret we find an upper bound on the number of times we select a suboptimal arm $i$:
\begin{equation} 
	\mathbb{E[}n^1_i(T)\mathbb{]}=\mathbb{E}\left[\sum_{t=1}^{T} 1_{a(t)=i}\right] =
	\mathbb{E}\left[\sum_{t=1}^{U} 1_{a(t)=i}\right] + \mathbb{E}\left[\sum_{t=U+1}^{T} 1_{a(t)=i}\right].
\end{equation}
Since $\mathbb{E}\left[\sum_{t=1}^{U} 1_{a(t)=i}\right] \leq U$, we can deduce 
\begin{equation}
	\mathbb{E[}n^1_i(T)\mathbb{]} \leq U + \mathbb{E}\left[\sum_{t=U+1}^{T} 1_{a(t)=i}\right], 
\end{equation}
as we saw in \eqref{lema3}, $1_{\{a(t)=i\}}=1$ requires that at least one of the three equations \eqref{3cond1}, \eqref{3cond2} and \eqref{3cond3} is true. Assume the last time that \eqref{3cond3} is true is at time  $\zeta$, hence   
\begin{equation}
	n^1_i(\zeta)\leq \frac{2\alpha}{ \triangle_i^2}\ln \zeta-n^\prime_i(\zeta),
\end{equation}
and since $\zeta$ is the last time that \eqref{3cond3} holds, we can upper bound $\mathbb{E[}n^1_i(T)\mathbb{]}$ by
\begin{equation} \label{3true}
		\mathbb{E[}n^1_i(T)\mathbb{]} 
		\leq \frac{2\alpha}{ \triangle_i^2}\ln \zeta-n^\prime_i(\zeta) 
		+ \mathbb{E}\left[\sum_{t=\zeta+1}^{T} 1_{\{\eqref{3cond1}\ or \ \eqref{3cond2}\ is\  true\ and\ \eqref{3cond3}\ is\ false \}}\right].
\end{equation}
According to the definition of $U$ in \eqref{U} and Equation \eqref{3true} we have 
\begin{equation}
	\mathbb{E[}n^1_i(T)\mathbb{]}  \leq U +\mathbb{E}\left[\sum_{t=U+1}^{T} 1_{\{\eqref{3cond1}\ or \ \eqref{3cond2}\ is\  true\ and\ \eqref{3cond3}\ is\ false\}}\right]
\end{equation} 
\begin{equation}
	\leq U+{\sum_{t=U+1}^{T}} \mathbb{P}[\mbox{\eqref{3cond1} is true}]+\mathbb{P}[\mbox{\eqref{3cond2} is true}].
\end{equation}
It suffices to bound the probability \eqref{3cond1} and \eqref{3cond2}:
\begin{equation}\label{eq28}
	\mathbb{P}\left[\mbox{\eqref{3cond1} is true}\right] = \sum_{n_i(t)} \mathbb{P}\left[\mbox{\eqref{3cond1} is true} | n_i(t)\right]\cdot \mathbb{P}\left[n_i(t)\right].
\end{equation}
Recall that
\begin{equation}\label{eq29}
	\mathbb{P}\left[\mbox{\eqref{3cond1} is true} | n_i(t)\right]\leq\frac{1}{t^\alpha}.
\end{equation}
Plugging \eqref{eq29} in \eqref{eq28} yields that
\begin{equation}\label{eq30}
	\mathbb{P}\left[\mbox{\eqref{3cond1} is true}\right] \leq \frac{1}{t^\alpha} \sum_{n_i(t)}  \mathbb{P}\left[n_i(t)\right] = \frac{1}{t^\alpha}.
\end{equation}

\noindent Then we take integral of $\frac{1}{t^{\alpha}}$ for $t$ from 1 to $T$, which is smaller than $\frac{\alpha}{2(\alpha-2)}$.
The same upper bound holds for \eqref{3cond2}.

Thus, the regret is
\begin{equation}
	R \leq  \sum_{i,\triangle_i>0}\left(\max \left \{\max_{t=1,..,T} \left\{\frac{2\alpha}{\triangle_i^2}\ln t-n^\prime_i(t)\right\}\triangle_i , 0  \right\} + \frac{\alpha}{\alpha-2}\right) 
\end{equation}
as desired.
\end{proof}

\subsection{$\UCBN$ on Complete Graphs}
In this section, we analyze the regret in a complete graph when all agents use $\UCBN$.
The following curious lemma will assist in the proof.

\begin{lem} \label{all}
Given a complete graph of agents, if all agents use $\UCBN$ with a deterministic common tie breaking scheme, then in every time step all agents select the same action.
\end{lem}
\begin{proof}
Since the graph is complete, all agents see the rewards of other agents at every time step; hence the sample means $\hat s_i(t)$ and number of samples $n_i(t)$ at time $t$ are the same for all agents. Furthermore, every agent selects an arm according to criteria \eqref{rule} and a common deterministic tie breaking rule. Therefore, the arm selected at time $t$ will be same for all agents.
\end{proof}

\begin{proof}[Proof of Corollary~\ref{allUCBN}]
Let
\begin{equation}
	U=\left[\frac{2\alpha \ln T}{b\cdot \triangle_i^2}\right].
\end{equation}
We bound the number of times action $i$ other than the best arm is selected. Following the proof of Theorem~\ref{thm:UCBN_ub}, 
\begin{equation} 
	\mathbb{E[}n^1_i(T)\mathbb{]} \leq
	U+{\sum_{t=U+1}^{T}} \mathbb{P}[\mbox{\eqref{3cond1} is true}]+\mathbb{P}[\mbox{\eqref{3cond2} is true}].
\end{equation}
The upper bound of the probabilities \eqref{3cond1} and \eqref{3cond2} are same as before.
Hence, the regret bound is
\begin{equation}\label{bound}
	R \leq \sum_{i,\triangle_i>0} \left(\frac{2\alpha}{b\cdot \triangle_i}\ln T+ \frac{\alpha}{\alpha-2}\right).
\end{equation}

\end{proof}

\subsection{Lower Bound}
The lower bound for $\UCB$ \cite{LR1985} is
\begin{equation}
	\lim_{T\longrightarrow +\infty} \inf
	\frac{R}{\ln T}\geq \sum_{i,\triangle_i>0} \frac{\triangle_i}{ kl(\mu_i,\mu^\star)}.
\end{equation}
Our proof again follows the same template.

\begin{proof}[Proof of Theorem~\ref{thm:UCBN_LB}]
As in \cite{LR1985}, we assume the rewards are drawn from a Bernoulli distribution. From their proof it follows that the expected number of times that a suboptimal arm must be selected in order to distinguish between best arm and other arms is at least 
\begin{equation}
	\mathbb{E[}n_i(T)\mathbb{]} +c_T(i) \geq (1+o(1))\frac{1-\eps}{1+\eps} \frac{\ln T}{kl(\mu_i,\mu^\star)}.
\end{equation}
where the second term is the information coming from the neighbors (that is, the number of times neighbors selected arm $i$ up to time $t$), $\mu^\star$ is the mean of the best arm, and
 $kl(p,q)$ is the Kullback-Leibler divergence between a Bernoulli variable with parameter p and a Bernoulli variable with parameter q, defined to be  
\begin{equation}
	kl(p,q)\defeq p\ln\left(\frac{p}{q}\right)+(1-p)\ln \left(\frac{1-p}{1-q}\right).
\end{equation}

As the number of rounds increases, $\eps$ can be taken to be smaller.
As $T$ goes to infinity, $\eps$ can be taken zero; as a result we can write the following lower bound for the regret
\begin{equation}
	\lim_{T\longrightarrow +\infty} \inf
	\frac{R +\sum_{i,\Delta_i>0} c_T(i)\Delta_i}{\ln T}\geq \sum_{i,\triangle_i>0} \frac{\triangle_i}{ kl(\mu_i,\mu^\star)}.
\end{equation}
Applying Pinkser's inequality yields the lower bound.
\end{proof}

\end{document}